\documentclass[11pt]{article}
\usepackage[margin=1in]{geometry}
\usepackage{amsmath,amssymb,amsthm,mathtools}
\usepackage{graphicx}
\usepackage{booktabs}
\usepackage{microtype}
\usepackage{pifont}
\newcommand{\cmark}{\ding{51}}
\newcommand{\xmark}{\ding{55}}
\usepackage{algorithm}
\usepackage{algpseudocode}
\usepackage[hidelinks]{hyperref}

\newtheorem{theorem}{Theorem}
\newtheorem{proposition}{Proposition}
\newtheorem{corollary}{Corollary}
\newtheorem{lemma}{Lemma}
\theoremstyle{definition}
\newtheorem{definition}{Definition}
\newtheorem{remark}{Remark}
\newtheorem{example}{Example}

\newcommand{\T}{\mathcal{T}}
\newcommand{\Ow}{\mathcal{O}}

\newcommand{\Adm}{\Gamma}
\newcommand{\Ker}{K}
\newcommand{\Real}{\mathbb{R}}
\newcommand{\Z}{Z}
\newcommand{\diag}{\mathrm{diag}}
\newcommand{\Tr}{\mathrm{Tr}}
\newcommand{\NDE}{\mathrm{NDE}}
\newcommand{\PN}{\mathrm{PN}}

\title{\textbf{WorldKernel: A World Model is the Coupling Kernel\\ of Admissible Possible Worlds}\\[4pt]
\large Admissibility on the Diagonal, Counterfactual Coupling on the Off-Diagonal,\\ and Ontology-Structured Bounds up to a Sly--Sun Counting Barrier}
\author{Fabio Rovai\\ \texttt{fabio@thetesseractacademy.com}}
\date{June 2026}

\begin{document}
\maketitle

\begin{abstract}
A common assumption holds that enough observational and interventional data, given to a strong enough predictor, is effectively enough. We report a recurring failure mode that contradicts it: across hundreds of structural causal models, on quantities the data identify a strong predictor and a Bayesian baseline succeed, but on \emph{unidentified} quantities, the couplings between counterfactual worlds, the predictor collapses to a point, on $28\%$ of models to one no valid model can produce, while the truth is an admissible interval more data never narrows. The gap is structural: prediction cannot represent uncertainty over counterfactual couplings. We cast a world model as a single positive semidefinite \emph{coupling kernel} $\Ker_E(T,T')$ over admissible worlds, whose diagonal is the ordinary posterior (what a predictor recovers) and whose \emph{off-diagonal} is the cross-world coupling it cannot, and which every counterfactual reads. The paper is the theory of that off-diagonal, organised as five questions about one object. \emph{Is it real?} Two states with identical posteriors differ on a cross-world query, and the off-diagonal coincides with the coupling that fixes counterfactuals. \emph{Can we bound it?} Its positive semidefiniteness is partial-identifying information the marginals lack; enforcing it (a semidefinite relaxation) bounds counterfactuals in polynomial time where the exact response-type program is intractable, returning the exact bound at scales that program cannot reach. \emph{Can logical structure sharpen it?} Ontology axioms enter as further constraints and tighten the bound by up to a third, propagating through positive semidefiniteness to couplings they never touch, structure flat causal-bounds tools cannot use. \emph{Can we acquire it?} Targeted ``scars'', constraints learned from encountered infeasibilities, close the gap up to $4\times$ faster than untargeted ones (the margin narrowing as both saturate), with a spectral-gap-restoration guarantee. \emph{When is it reachable?} Accessing it is approximate counting of the admissible worlds, tractable below the Sly--Sun threshold and inapproximable above; we do not claim to beat the worst case, and we measure exactly where the phase-extremal recourse works (few-phase) and provably stalls (the glass). A world model is the coupling kernel of what can consistently be; its off-diagonal is the part prediction will never supply.
\end{abstract}

\section{Introduction}

We began not from a definition but from a failure. A widespread working assumption is that observational and interventional data, in sufficient quantity and given to a sufficiently strong predictor, suffice to reason about a domain. Section~\ref{sec:battery} reports a battery of three hundred structural causal models on which this assumption breaks in a specific, repeatable way. Each model is queried for a counterfactual. When the query is identified by observational and interventional data, a strong language model, a Bayesian structural-causal baseline, and the representation of this paper all answer it. When the query is \emph{not} identified, the correct answer is properly an interval of admissible values, and the systems diverge: the Bayesian baseline emits one feasible point that is frequently far from truth and whose sign is unstable, while the language model emits a point that on $28\%$ of models is not even feasible, a value no structural causal model can produce. More data does not help, because the unidentified quantity, the coupling between counterfactual worlds, is not a function of any observational or interventional distribution. The failure is therefore not statistical but representational: \emph{prediction cannot represent uncertainty over counterfactual couplings.}

This paper develops the minimal object that can represent that uncertainty. A world model is a single positive semidefinite \emph{coupling kernel}
\[
\Ker_E(T,T') = \langle T \mid \rho_E \mid T' \rangle
\]
over complete admissible worlds $T,T'$, conditioned on evidence $E$. The diagonal $\Ker_E(T,T)=\mu_E(T)$ is the classical posterior over admissible worlds, exactly what observational and interventional data, hence any predictor, can recover. The off-diagonal $\Ker_E(T,T')$ links alternative worlds, and it is exactly the cross-world coupling the predictor cannot represent. Time links stages within a world; this kernel links worlds. The posterior is only the diagonal shadow of the kernel, and the failure mode above is precisely the absence of the off-diagonal.

\paragraph{One object.} This paper is about a single object: the \emph{off-diagonal} of the world kernel, the cross-world coupling $\Ker_E(T,T')$ that links alternative worlds. A predictor recovers only the diagonal, the posterior over what is the case; the off-diagonal is the one thing it cannot represent, and it is exactly what a counterfactual reads. Everything else in the paper is the theory of that object, organised as the questions any account of it must answer.
\begin{enumerate}
\item \emph{Is it real, or just prediction in disguise?} Real and separate. Worlds can be predictively identical yet admissibility-distinct (Section~\ref{sec:diagonal}), and two states with identical posteriors can differ on a cross-world query (Section~\ref{sec:kernel}); empirically the off-diagonal coincides with the coupling that fixes counterfactuals, which strong predictors collapse (Section~\ref{sec:empirics}).
\item \emph{Can we bound it from data?} Yes, and the kernel's own structure does the work. Positive semidefiniteness is partial-identifying information the marginals discard, tightening counterfactual bounds in polynomial time where the exact program is intractable (Section~\ref{sec:psdbounds}).
\item \emph{Can logical structure sharpen it?} Yes. An ontology's axioms enter as further constraints and tighten the bound by up to a third, propagating to couplings they never touch, structure flat causal-bounds tools cannot use (Section~\ref{sec:ontobounds}).
\item \emph{Can we acquire it from experience?} Yes. Targeted ``scars'', constraints learned from encountered infeasibilities, close the identifiability gap up to $4\times$ faster than untargeted ones (narrowing as both saturate), with a Cheeger guarantee (Section~\ref{sec:onlinescar}).
\item \emph{When is it reachable?} Sharply bounded. Accessing it is approximate counting of the admissible worlds, tractable below the Sly--Sun threshold and inapproximable above it, the order parameter $(d-1)\eta$ crossing $1$ between constraint-graph degree $5$ and $6$ (Section~\ref{sec:barrier}).
\end{enumerate}
One object, five questions, not five ideas. Section~\ref{sec:intelligence} states what the object buys: intelligence as closure-preserving counterfactual competence.

\paragraph{What is new and what is borrowed.} We do not claim a new sampling-to-counting theorem, a new shielding method, or a new partial-identification algorithm. The contribution is the synthesis: the identification of a single object (the positive semidefinite kernel), the proof that its off-diagonal is a distinct and counterfactually meaningful direction, the demonstration that its semidefinite structure tightens counterfactual bounds beyond the marginals on instances where the exact response-type program is intractable, the placement of a known approximate-counting barrier as the boundary for full off-diagonal access, and the empirical demonstration that the off-diagonal is load-bearing for causal reasoning where strong predictors collapse.

\section{Possible worlds, queries, and the insufficiency of prediction}
\label{sec:diagonal}

\begin{definition}[Worlds and query classes]
Let $\Omega$ be a set of complete admissible worlds (or histories). A query class $Q$ is a set of functions $q:\Omega\to Y_q$. We distinguish predictive queries $Q_{\mathrm{pred}}$ (e.g.\ $p(o_{t+1}\mid o_{\le t},a_t)$), entailment queries $Q_{\mathrm{ent}}$ ($\Ow\models\varphi$), admissibility queries $Q_{\mathrm{adm}}$ ($a\in\Adm_\Ow(s)$), intervention queries $Q_{\mathrm{caus}}$ ($P(Y\mid do(a),E)$), and kernel queries $Q_{\mathrm{ker}}$ ($\Ker_E(\omega,\omega')$). The world query class is $Q_{\mathrm{world}}=Q_{\mathrm{pred}}\cup Q_{\mathrm{ent}}\cup Q_{\mathrm{adm}}\cup Q_{\mathrm{caus}}\cup Q_{\mathrm{ker}}$.
\end{definition}

\begin{definition}[Query equivalence and completeness]
For a class $Q$, set $\omega\sim_Q\omega'$ iff $q(\omega)=q(\omega')$ for all $q\in Q$. A representation $\phi:\Omega\to Z$ is $Q$-complete if for each $q\in Q$ there is a decoder $f_q$ with $q=f_q\circ\phi$.
\end{definition}

\begin{theorem}[Query quotient]
\label{thm:quotient}
The quotient map $\pi_Q:\Omega\to\Omega/\!\sim_Q$ is $Q$-complete, and any $Q$-complete $\phi$ factors through it: $\phi(\omega)=\phi(\omega')\Rightarrow\omega\sim_Q\omega'$. Hence $\Omega/\!\sim_Q$ is the coarsest representation sufficient for $Q$.
\end{theorem}
\begin{proof}
Define $f_q([\omega])=q(\omega)$, well defined since $[\omega]=[\omega']$ implies $q(\omega)=q(\omega')$; then $q=f_q\circ\pi_Q$. Conversely if $\phi$ is complete and $\phi(\omega)=\phi(\omega')$ then $q(\omega)=f_q(\phi(\omega))=f_q(\phi(\omega'))=q(\omega')$ for all $q$, so $\omega\sim_Q\omega'$.
\end{proof}

\begin{corollary}[Formal intelligence criterion]
If $Q_{\mathrm{pred}}\subsetneq Q_{\mathrm{world}}$ strictly on a domain, a predictive representation can be complete for prediction yet incomplete for world understanding. World competence requires a representation at least as fine as $\Omega/\!\sim_{Q_{\mathrm{world}}}$.
\end{corollary}

\begin{theorem}[Predictive insufficiency]
\label{thm:predins}
There exist worlds $W_0,W_1$ that are predictively equivalent but not world-equivalent. No representation trained only to be prediction-complete is guaranteed complete for world understanding.
\end{theorem}
\begin{proof}
Take a deterministic arena with states $S$, actions $A$, transition $T:S\times A\to S$, and a region $L\subset S$ that is physically traversable. Let $W_0=(S,A,T,O,\Ow_0)$ and $W_1=(S,A,T,O,\Ow_1)$ share the same dynamics and observation map, with $\Ow_0$ entailing no restriction on $L$ and $\Ow_1$ entailing that transitions into $L$ are forbidden. All predictive conditionals over observations coincide, so $W_0,W_1$ are predictively equivalent. Choose $s,a$ with $T(s,a)\in L$. Then $a\in\Adm_{\Ow_0}(s)$ but $a\notin\Adm_{\Ow_1}(s)$, so the admissibility query separates them. By Theorem~\ref{thm:quotient} any admissibility-complete representation must separate them, and a prediction-only representation has no information to do so.
\end{proof}

\paragraph{The diagonal is certifiable.} The diagonal of the kernel is admissibility, and it can be maintained with a guarantee independent of the learned model. Let $\Phi_\Ow\subseteq S$ be the entailed forbidden set and $V_\Ow$ the entailed transition-violation predicate.

\begin{theorem}[Closure-preserving horizon soundness]
\label{thm:soundness}
Let a projector $g_\Ow$ be single-step sound: for $s\notin\Phi_\Ow$ and any proposed action $a$, $g_\Ow(s,a)\notin\Phi_\Ow$ and $\neg V_\Ow(s,g_\Ow(s,a))$. If $s_0\notin\Phi_\Ow$ and $s_{t+1}=g_\Ow(s_t,a_t)$ for an arbitrary, possibly arbitrarily inaccurate action sequence, then $s_t\notin\Phi_\Ow$ and $\neg V_\Ow(s_t,s_{t+1})$ for all $t$.
\end{theorem}
\begin{proof}
Induction. Base $s_0\notin\Phi_\Ow$. If $s_t\notin\Phi_\Ow$, single-step soundness gives $g_\Ow(s_t,a_t)\notin\Phi_\Ow$ and $\neg V_\Ow(s_t,g_\Ow(s_t,a_t))$, and $s_{t+1}=g_\Ow(s_t,a_t)$.
\end{proof}
The point is that this competence is independent of the learned model $F_\theta$: an arbitrarily inaccurate predictor still cannot leave the admissible set. Empirically, in a permissive-lava arena that we use as the admissibility witness, an unshielded predictor reaches goals by entering a forbidden region on $36.4\%$ of episodes ($91/250$), while the diagonal projector, using the same wall-blind learned model, commits zero forbidden transitions over $250$ episodes (Figure~\ref{fig:lavatraj}). The same OWL-derived transition grammar improves competence and not only safety: transplanted onto an unmodified learned planner it raises cross-room success across a range of wall topologies without retraining. The entailed constraint is not a safety patch bolted on, it is information the predictor lacked.

\begin{figure}[t]
\centering
\includegraphics[width=0.72\textwidth]{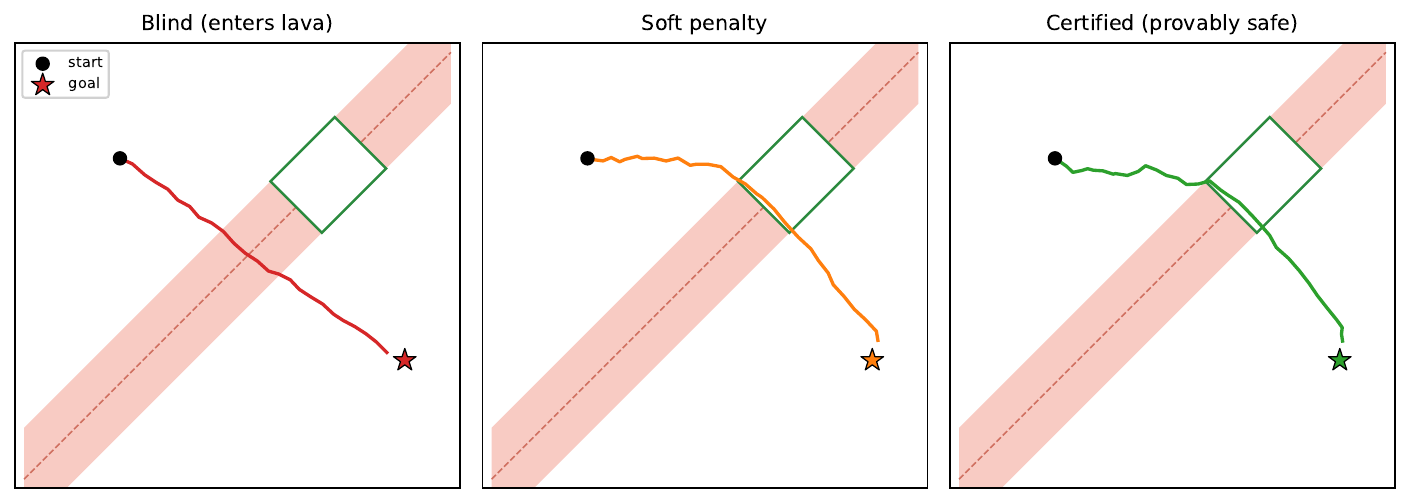}
\caption{Permissive-lava arena. The unshielded planner reaches the goal by crossing the forbidden band; the diagonal projector, on the same learned model, threads the entailed door and never enters the band.}
\label{fig:lavatraj}
\end{figure}

\paragraph{The diagonal curves space.} Admissibility is not only logical, it is geometric, which is the cleanest way to see that the entailed structure is a property of the action space and not a post-hoc filter.

\begin{proposition}[Entailed geometry]
\label{prop:geo}
Let $\Phi_\Ow\subset X\subset\Real^2$ be the entailed forbidden set and $g_M=e^{2M\mathbf{1}_{\Phi_\Ow}}\delta$ the conformal metric with barrier height $M$. Any limit as $M\to\infty$ of minimizing $g_M$-paths from $a$ to $b$ is a shortest path among rectifiable paths avoiding $\Phi_\Ow$, and the Gaussian curvature $K=-e^{-2\phi}\Delta\phi$ concentrates on $\partial\Phi_\Ow$.
\end{proposition}
\begin{proof}[Proof sketch]
A path avoiding $\Phi_\Ow$ has $g_M$-length equal to its Euclidean length, independent of $M$. A path spending Euclidean length $\ell_M$ inside $\Phi_\Ow$ has $g_M$-length $\ge e^{M\ell_M}$, so minimality forces $\ell_M\to 0$; on the complement of $\Phi_\Ow$ the metric is Euclidean, so the limit minimizes Euclidean length among avoiding paths.
\end{proof}

Forbidden regions become curvature ridges and entailed openings become flat passes: knowledge has become metric.

\section{The kernel and its off-diagonal}
\label{sec:kernel}

Let $\Omega=\T_E$ be the finite set of complete admissible worlds under evidence $E$, with $\mathcal{H}=\ell^2(\T_E)$ and basis $\{\,|T\rangle\,\}$. A weighted generator assigns $w_E(T)\ge 0$ with normalizer $\Z_E=\sum_T w_E(T)$ and posterior $\mu_E(T)=w_E(T)/\Z_E$.

\begin{definition}[World kernel]
A world kernel is a density operator $\rho_E$ with matrix elements $\Ker_E(T,T')=\langle T\mid\rho_E\mid T'\rangle$ whose diagonal is the classical posterior, $\Ker_E(T,T)=\mu_E(T)$. The coherent posterior state is $|\Pi_E\rangle=\sum_T\sqrt{\mu_E(T)}\,|T\rangle$, with pure kernel $\Ker_E^{\mathrm{pure}}(T,T')=\sqrt{\mu_E(T)\mu_E(T')}$.
\end{definition}

The diagonal is ordinary Bayesian uncertainty. The off-diagonal says two worlds are not merely alternatives: they can interfere, be coherently transformed, and be compared under non-classical operations. We first confirm the kernel is conservative, then that the off-diagonal is not.

\begin{lemma}[Classical shadow]
For any classical event $A\subseteq\T_E$ with $P_A=\sum_{T\in A}|T\rangle\langle T|$, $\Tr(\rho_E P_A)=\sum_{T\in A}\Ker_E(T,T)$. If evidence $B$ is observed and $\rho_B=P_B\rho_E P_B/\Tr(\rho_E P_B)$, then on the diagonal $\Ker_B(T,T)=\Ker_E(T,T)/\sum_{T'\in B}\Ker_E(T',T')$ for $T\in B$ and $0$ otherwise.
\end{lemma}
\begin{proof}
$P_A$ is diagonal in the world basis, so $\Tr(\rho_E P_A)=\sum_{T\in A}\langle T|\rho_E|T\rangle$. For the update, $\langle T|P_B\rho_E P_B|T\rangle=\mathbf{1}[T\in B]\Ker_E(T,T)$; divide by $\Tr(\rho_E P_B)=\sum_{T'\in B}\Ker_E(T',T')$.
\end{proof}
So measuring the diagonal reproduces Bayesian conditioning: the kernel extends, it does not replace, probabilistic ontology.

\begin{example}[Conservativity is exact]
Let a generator propose three complete taxonomy worlds with evidence weights $w(T_0)=6,\,w(T_1)=3,\,w(T_2)=1$, so $\Z=10$ and $\mu=(0.6,0.3,0.1)$. Suppose an intervention certifier rules out $T_2$, accepting the event $B=\{T_0,T_1\}$. Classical conditioning gives $\mu(T_0\mid B)=6/9=2/3$ and $\mu(T_1\mid B)=1/3$. The kernel projection $\rho_B=P_B\rho P_B/\Tr(\rho P_B)$ has diagonal $(\tfrac23,\tfrac13,0)$, identical. The off-diagonal changes nothing classical marginal observers can see; it adds structure only off the diagonal.
\end{example}

\begin{proposition}[Same diagonal, different kernel]
\label{prop:sep}
There exist two world kernels with the same posterior diagonal but different answers to a coherent linkage query.
\end{proposition}
\begin{proof}
Let $\T=\{T_0,T_1\}$. Take $\rho_{\mathrm{diag}}=\tfrac12|T_0\rangle\langle T_0|+\tfrac12|T_1\rangle\langle T_1|$ and $\rho_{\mathrm{coh}}=|\Pi\rangle\langle\Pi|$ with $|\Pi\rangle=(|T_0\rangle+|T_1\rangle)/\sqrt2$. Both have diagonal $(\tfrac12,\tfrac12)$. For the projector $P_+=|\Pi\rangle\langle\Pi|$, $\Tr(\rho_{\mathrm{coh}}P_+)=1$ while $\Tr(\rho_{\mathrm{diag}}P_+)=\tfrac12$. Identical posteriors, different kernels.
\end{proof}

This is the conceptual pivot: ``the posterior is the whole story'' is incomplete, because the posterior is only the diagonal. The full object is the kernel, and Section~\ref{sec:empirics} shows the off-diagonal is precisely the cross-world coupling that fixes counterfactuals, the quantity a posterior cannot carry.

As a density matrix the kernel stays computable: the Shannon entropy of its diagonal $H(\mu_E)$ is a remaining-uncertainty scalar that falls to zero as evidence pins a unique world; conditioning is projection-and-renormalize at cost $O(|\mathcal{M}(\T)|^2)$; and the world history is realized concretely as a deterministic event-graph substrate~\cite{substrates}, with Open Ontologies~\cite{openont} supplying the grammar and CIVeX~\cite{civex} certifying admissible interventions. Its positive semidefiniteness is itself identifying information the marginals lack, which Section~\ref{sec:psdbounds} turns into counterfactual bounds tighter than the marginals allow, in polynomial time, where the exact program is intractable.

\section{Empirical validation: the off-diagonal is load-bearing}
\label{sec:empirics}

We now show, on real computations and against a strong language model used as the predictor baseline, that the off-diagonal carries information the diagonal cannot, and that it is exactly the counterfactual. First we make the kernel-counterfactual correspondence a theorem rather than a metaphor.

\begin{theorem}[Kernel--counterfactual isomorphism]
\label{thm:iso}
Fix a binary treatment with potential outcomes $(Y_0,Y_1)$ and $v=(Y_0,Y_1)^\top$. For a counterfactual law $P$ on $\{0,1\}^2$ let $M(P)=\mathbb{E}_P[vv^\top]$, a $2\times2$ PSD matrix with $M_{ii}=P(Y_i{=}1)$ and $M_{01}=P(Y_0{=}1,Y_1{=}1)$. Then $P\mapsto M(P)$ is a bijection between counterfactual equivalence classes (laws agreeing on every counterfactual query) and the set
\[
\mathcal{K}=\Big\{\,M\succeq 0:\ M_{ii}\in[0,1],\ M_{01}\in\big[\max(0,M_{00}{+}M_{11}{-}1),\ \min(M_{00},M_{11})\big]\Big\}.
\]
Under this isomorphism: (i) the rung-1/2 (observational and interventional) identifiable data are exactly $\diag M$; (ii) every counterfactual query is a function of $M$, for instance $\PN=(M_{11}-M_{01})/M_{11}$; (iii) two laws with equal $\diag M$ but different off-diagonal $M_{01}$ are observationally and interventionally indistinguishable yet counterfactually distinct.
\end{theorem}
\begin{proof}
The four cells recover from $M$ by $p_{11}=M_{01}$, $p_{10}=M_{00}-M_{01}$, $p_{01}=M_{11}-M_{01}$, $p_{00}=1-M_{00}-M_{11}+M_{01}$, so $P\mapsto(M_{00},M_{11},M_{01})$ is affine and invertible on the simplex; nonnegativity of the four cells is exactly the Fr\'echet--Hoeffding range on $M_{01}$, and for $\{0,1\}$ variables $M=\mathbb{E}[vv^\top]\succeq 0$ holds automatically, so the image is precisely $\mathcal{K}$. Counterfactual queries are functions of $P$, hence of $M$; the necessity formula is $\PN=P(Y_0{=}0\mid Y_1{=}1)=p_{01}/M_{11}$. Under randomized $X$ the rungs $1$ and $2$ identify the marginals $M_{00},M_{11}=\diag M$ and place no constraint on $M_{01}$ beyond $\mathcal{K}$, giving (i) and (iii).
\end{proof}

Algorithm~\ref{alg:witness} is the constructive content of Theorem~\ref{thm:iso} and Proposition~\ref{prop:sep} at once, and is exactly the witness computed in Section~\ref{sec:cfquery}: it fixes the diagonal, varies only the off-diagonal across its Fr\'echet range, and reads off two distinct counterfactuals that no diagonal (predictor) answer can match.

\begin{algorithm}[t]
\caption{Off-diagonal witness and probability of necessity (proves Thm.~\ref{thm:iso}, Prop.~\ref{prop:sep})}
\label{alg:witness}
\begin{algorithmic}[1]
\Require marginals $r_0=P(Y_0{=}1)$, $r_1=P(Y_1{=}1)$ \Comment{the rung-1/2 identifiable diagonal}
\State $c_{\min}\gets\max(0,r_0{+}r_1{-}1)$, \ $c_{\max}\gets\min(r_0,r_1)$ \Comment{Fr\'echet--Hoeffding range of $M_{01}$}
\For{coupling $c\in\{\,c_{\max}\ (\text{monotonic}),\ r_0 r_1\ (\text{independent})\,\}$}
  \State $M\gets\begin{psmallmatrix}r_0 & c\\ c & r_1\end{psmallmatrix}$ \Comment{same diagonal $(r_0,r_1)$ for every $c$}
  \State \textbf{assert} $M\succeq 0$ and all four cells $p_{ij}\ge 0$ \Comment{$M\in\mathcal{K}$}
  \State $\PN(c)\gets (M_{11}-M_{01})/M_{11}=(r_1-c)/r_1$
\EndFor
\State \textbf{predictor:} given only $\diag M=(r_0,r_1)$, returns a single point $\hat p$
\State \Return $\PN(c_{\max})\neq\PN(r_0r_1)$, yet $\hat p$ cannot equal both \Comment{diagonal underdetermines the counterfactual}
\end{algorithmic}
\end{algorithm}

\begin{remark}[General case]
For arbitrary intervention sets and mediation the worlds are the potential outcomes under each regime, the kernel is the full response-type law (a density operator whose diagonal in the response-type basis is the counterfactual law), and the rung-1/2 data are a fixed set of linear functionals of it. The identified counterfactual class is then a polytope, the feasible region of the linear program of Section~\ref{sec:battery}; for a single binary treatment the $2\times2$ moment matrix $M$ is its sufficient invariant, but with three or more worlds $M$ captures only the pairwise structure (Remark~\ref{rem:moments}). In every case the counterfactual content is exactly the off-diagonal, and rung-1/2 data fix only the diagonal.
\end{remark}

\begin{remark}[What $M$ misses: the moment hierarchy]
\label{rem:moments}
The moment matrix carries only first and second cross-world moments, so any counterfactual that depends on a third-or-higher-order joint of worlds is separated by the full kernel $\rho$ yet invisible to $M$. With three potential outcomes $(Y_0,Y_1,Y_2)$, the uniform law on the even-parity triples $\{000,011,101,110\}$ and the uniform law on the odd-parity triples $\{111,100,010,001\}$ have identical marginals ($\tfrac12$) and identical pairwise couplings ($\tfrac14$), hence identical $M$, yet the three-way query $P(Y_0{=}Y_1{=}Y_2{=}1)$, recovery under all three arms, is $0$ versus $\tfrac14$. So $M$ is the first level of the moment (Lasserre) hierarchy: exact for queries that are functions of at most second moments, and for the binary two-world case where it is the full law, but a valid \emph{relaxation} above, the slack being the un-tracked higher-order coupling. This is precisely why the polynomial-time bound of Section~\ref{sec:psdbounds} is tight on second-order aggregates (the $40/40$ exact match) and a strict outer bound on higher-order queries: on the three-way query $P(Y_0{=}Y_1{=}Y_2{=}1)$ the degree-2 bound is strictly looser than the exact response-type program, by a closed-form slack $\max(0,\min(\min_{ij}s_{ij},\min_i d_i)-\hat P_0)$ ($\hat P_0=1-\sum_i d_i+\sum_{ij}s_{ij}$) that matches the program to machine precision and is positive on $24\%$ of random instances, exactly the untracked $E[Y_iY_jY_k]$ (Figure~\ref{fig:momentgap}). The hierarchy has \emph{three distinct regimes}, which it is worth not conflating into one wall. \emph{(i)} The degree-2 relaxation is polynomial and exact only for $\le$second-order queries. \emph{(ii)} The exact top, the full response-type law, is exact counting and is \#P-hard: an \emph{unconditional} $2^m$ dimensional blowup, hard for every instance. \emph{(iii)} \emph{Approximate} full access above the Sly--Sun degree (Section~\ref{sec:barrier}) is NP-hard unless $\mathrm{NP}=\mathrm{RP}$: a \emph{conditional} inapproximability about the \emph{approximate} object, biting only above degree $5/6$. Climbing to a degree-$k$ moment matrix separates $k$-way structure at polynomial cost for fixed $k$; the cheap floor, the \#P exact top, and the Sly--Sun-barred approximate top are three different statements.
\end{remark}

\begin{remark}[The coherence is inert by construction]
\label{rem:coherence}
The inertness is structural, not contingent. A counterfactual event is a subset $\mathcal{E}$ of the response-type space, so its probability is the expectation of an observable \emph{diagonal in the world basis},
\[
P(\mathcal{E})=\sum_{T\in\mathcal{E}}\langle T|\rho|T\rangle .
\]
The entire counterfactual law is therefore the diagonal of the response-type density operator, and the coherence $\langle T|\rho|T'\rangle$ for $T\neq T'$ is orthogonal to the counterfactual algebra: no counterfactual event corresponds to a projector that reads it (only a projector onto a superposition of distinct worlds does, and that is not an event in potential-outcome space). The coherence does not happen to be inert; it is \emph{forced} to be, and that is itself a result, there is a degree of freedom in the object that provably no counterfactual query can see. Proposition~\ref{prop:sep}'s coherent witness is thus illustrative, while the operational separation is the classical one of Theorem~\ref{thm:iso}.
\end{remark}

\begin{remark}[Two matrices, one affine map: fixing the convention]
\label{rem:convention}
``Counterfactuals are the off-diagonal'' and ``counterfactuals are the diagonal'' are both true, of two different objects, and we fix the convention so the language does no rhetorical work across the gap. In the response-type density operator the full counterfactual law sits on the \emph{diagonal} (the joint over complete configurations) and the inert coherence sits off it. In the moment matrix $M=\mathbb{E}[vv^\top]$ the marginals (rung-1/2, what a predictor recovers) sit on the \emph{diagonal} and the cross-world coupling $M_{ij}$ sits off it. The two matrices are related by an invertible affine map, so the same probability of necessity is a diagonal read of one and an off-diagonal read of the other; ``diagonal'' is object-relative. We take the \emph{classical response-type law} as the primitive, summarized at each order by its moment matrices, with degree-2 $M$ the cheap floor and the full law the \#P/Sly--Sun-barred top (Remark~\ref{rem:moments}); the moment tower interpolates, and the coherence is a separate set of numbers in neither the marginals nor the tower, the one no query touches. Throughout, then, ``kernel'' names this classical object and the load-bearing ``off-diagonal'' is $M_{ij}$, the cross-world coupling. With coherence flagged inert, the three regimes become a single classical axis read at increasing order, with no quantum residue.
\end{remark}

\subsection{A counterfactual is an off-diagonal query}
\label{sec:cfquery}

Consider a binary treatment $X$ and outcome $Y$. The two potential-outcome worlds are $T_0$ (untreated, $Y_0$) and $T_1$ (treated, $Y_1$). The diagonal of the kernel is the pair of marginals $P(Y_0),P(Y_1)$, which are exactly what observational and interventional data reveal. The off-diagonal $\Ker(T_0,T_1)$ is the cross-world joint $P(Y_0,Y_1)$, the coupling between the two worlds, which no single intervention probes. The probability of necessity $\PN=P(Y_0=0\mid X=1,Y=1)$ reads this off-diagonal.

We construct two models with identical marginals $P(Y_0{=}1)=0.5$, $P(Y_1{=}1)=0.7$ (hence identical observational distribution under randomized $X$ and identical interventional $P(Y\mid do(X))$, average causal effect $0.2$ in both), but different cross-world coupling: model $A$ maximally positively coupled (monotone, treatment never hurts), model $B$ with independent potential outcomes. They give $\PN_A=0.286$ and $\PN_B=0.500$. We posed the question to a strong language model (Claude, headless, given the full observational and interventional tables) five times; it returned $0.40$ on every call, the cross-world-independence point estimate. A single number cannot equal both $0.286$ and $0.500$: the predictor collapses two worlds that the off-diagonal separates (Figure~\ref{fig:witness}). The LLM is not wrong so much as committed to one coupling without warrant; the separating information is the off-diagonal it does not carry.

\begin{figure}[t]
\centering
\includegraphics[width=0.95\textwidth]{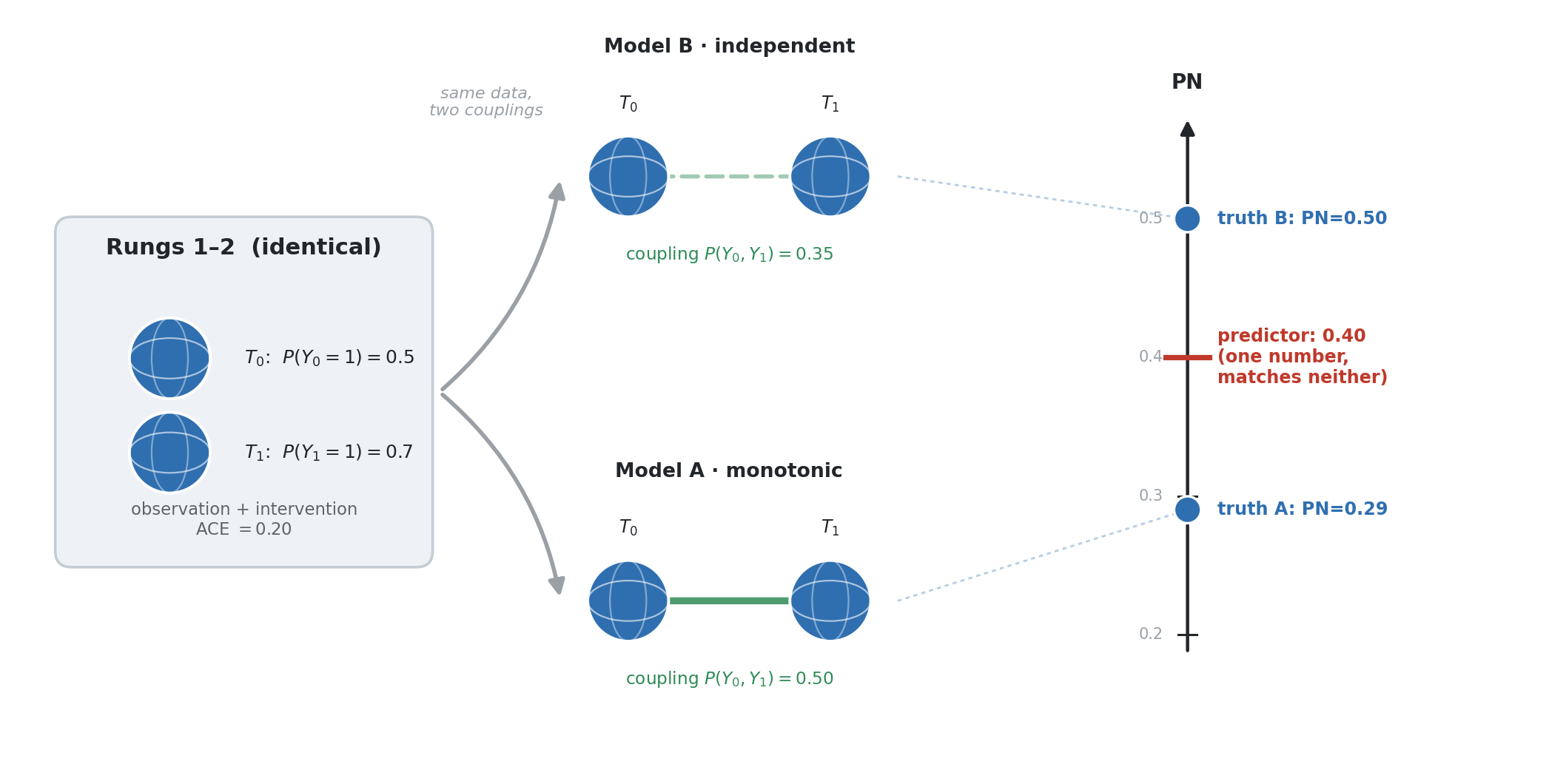}
\caption{The off-diagonal witness. Both models share rungs $1$ and $2$ (the same two worlds $T_0,T_1$, the same marginals and $\mathrm{ACE}=0.20$, left box), then fork on the cross-world coupling thread alone: monotonic (Model A) versus independent (Model B). That single difference moves the counterfactual from $\mathrm{PN}=0.29$ to $0.50$ (blue, the off-diagonal/kernel values, right). A strong predictor given all identifiable data returns one number, $0.40$ (red), matching neither. This is Proposition~\ref{prop:sep} on Pearl's ladder: the off-diagonal is the rung-$3$ information.}
\label{fig:witness}
\end{figure}

\subsection{Scaling to mediation: the sign of an effect is unidentified}

We scale to the canonical multivariable case, mediation $X\to M\to Y$, with the natural direct effect $\NDE=P(Y_{1,M_0}=1)-P(Y_{0,M_0}=1)$, a nested cross-world counterfactual that holds the mediator at its untreated value while setting the outcome's treatment on. Using a response-function representation over $64$ types, we fix every quantity an experiment can measure (the per-arm distribution of $M$, the in-world joint of $(M,Y)$ under each arm, and all controlled $P(Y\mid do(X),do(M))$) and compute the identified interval of $\NDE$ by linear programming over the response-type polytope. The interval is $[-0.381,+0.187]$, width $0.568$, and it \emph{spans zero}: the same complete experiment is consistent with the direct effect being harmful or helpful, and only the off-diagonal coupling decides the sign (Figure~\ref{fig:mediation}). The two endpoint models reproduce byte-identical rungs $1$ and $2$. The language model returned $-0.09$ on every call, one point inside the interval. The honest statement is not that the model is wrong, it is that the quantity is unidentified and the model hides this; computing the exact $0.568$-wide interval requires the polytope program, which is the kernel's job, not the predictor's.

\begin{figure}[t]
\centering
\includegraphics[width=0.7\textwidth]{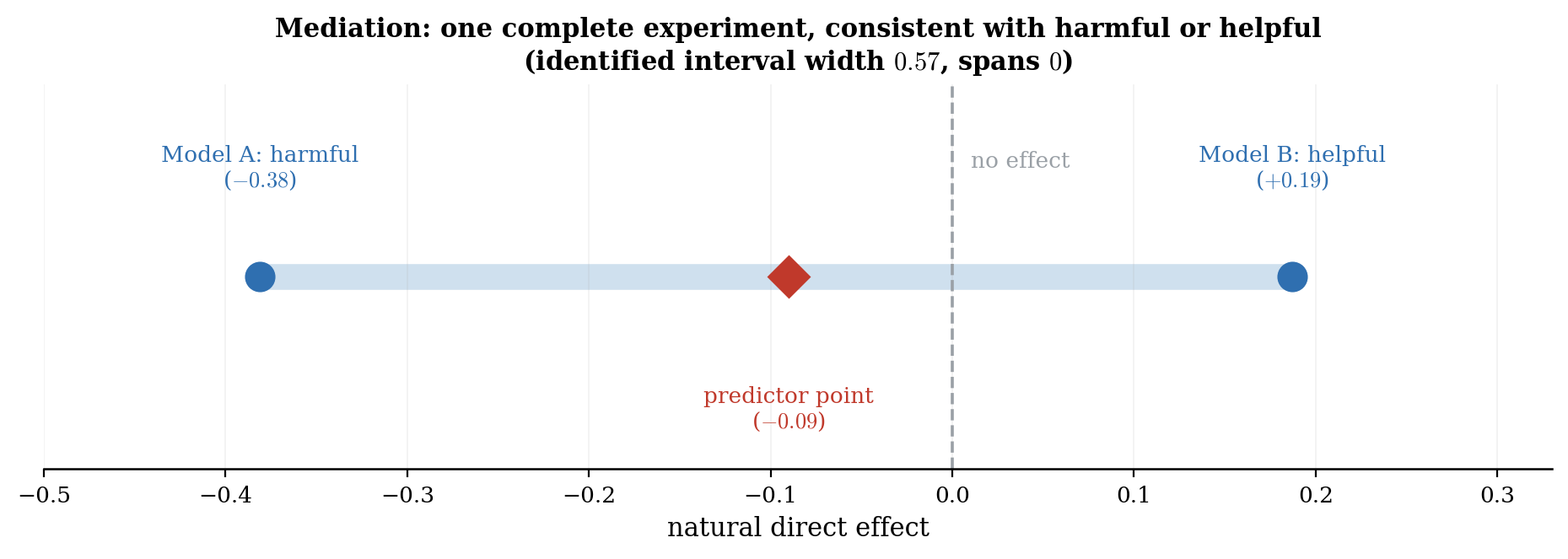}
\caption{Mediation $X\to M\to Y$. Fixing everything an experiment measures leaves the natural direct effect unidentified over an interval that spans zero. The endpoint models (blue) share identical rungs $1$ and $2$; the predictor (red) commits to one coupling.}
\label{fig:mediation}
\end{figure}

\subsection{The barrier meets the measurement on one axis}

Consider an aggregate counterfactual over a population of units in which units sharing a cause cannot both be deemed necessary, a mutual-exclusion constraint graph $G$ of degree $d$. The admissible joint necessity-attributions are exactly the independent sets of $G$, and the aggregate quantity (for instance the expected number of units for whom treatment was necessary) is a sum of hard-core occupation marginals. Belief propagation computes those marginals correctly below the Sly--Sun threshold (correlation decay) and fails above it. Sweeping the degree $d$ (at fugacity $\lambda=1$, the odds form of a per-unit probability of necessity $0.5$, which keeps the threshold at the clean value below), the relative error of the belief-propagation estimate of the aggregate counterfactual against exact enumeration tracks the cavity order parameter $(d-1)\eta$, both turning at $d_c\approx 5.14$ (the Sly--Sun implementation of Figure~\ref{fig:slysun}). This is the single figure where the encoding (independent-set count), the barrier (Sly--Sun $\eta$), and the measurement (the counterfactual aggregate) appear on one axis: what the off-diagonal decides, and when it is reachable. We are deliberate about the claim: at finite $n$ this is a monotonic correspondence, not a measured cliff, and it illustrates the proved asymptotic transition of Section~\ref{sec:barrier} rather than independently proving it.

\subsection{A structural-causal-model battery}
\label{sec:battery}
To move past single instances we draw $300$ random structural causal models per query type (direct effect, query $\PN$; mediation, query $\NDE$), each a full counterfactual law with computable ground truth. The principled comparison is between a \emph{diagonal-only} Bayesian SCM baseline (assuming cross-world independence, the standard naive default, which uses only the marginals) and the \emph{full kernel} (reporting the identified interval, and, given the coupling, the exact value); this isolates exactly the off-diagonal. As a secondary probe, not a partial-identification method, we also query a frontier \emph{language model} on the rung-1/2 data, only to show that a strong predictor inherits the same collapse rather than as a baseline any practitioner would use. The task is the one causal-inference researchers care about: recover the admissible counterfactual \emph{class}, not a point.

\begin{table}[t]
\centering
\begin{tabular}{lccc}
\toprule
system & direct effect ($\PN$) & mediation ($\NDE$) & reports the class? \\
\midrule
predictor (LLM) & err $0.23$; infeasible $28\%$ & not run & no (point, often infeasible) \\
diagonal-only (indep.\ SCM) & err $0.18$ & sign wrong $27\%$ & no (single point) \\
full kernel (interval) & covers $100\%$ & covers $100\%$ & yes (identified set) \\
full kernel (oracle coupling) & exact ($3\times10^{-17}$) & exact ($0$) & yes \\
\bottomrule
\end{tabular}
\caption{SCM battery, $300$ random models per query. ``Infeasible'' means the answer lies outside the mathematically possible interval (no valid SCM produces it). ``Covers'' is the fraction of models whose true counterfactual lies in the reported interval. All entries were locked by passing hard assertions (T1--T5 in the text).}
\label{tab:battery}
\end{table}

The pattern is uniform and was locked by hard assertions, all passing. The full-kernel interval covers the truth on $100\%$ of models and, given the coupling, is exact to machine precision: the off-diagonal is the sufficient statistic (T1, T2). The diagonal-only baseline is always feasible but commits to a point far from truth (mean $\PN$ error $0.18$) and gets the \emph{sign} of the mediation effect wrong on $27\%$ of models (T4): interventionally identical data, opposite conclusions. The predictor not only errs (mean $\PN$ error $0.23$) but returns \emph{infeasible} answers on $28\%$ of models (T3), point estimates no valid SCM can produce. It does not guess imprecisely, it hallucinates a point where the truth is a set (Figure~\ref{fig:battery}).

\begin{figure}[t]
\centering
\includegraphics[width=0.96\textwidth]{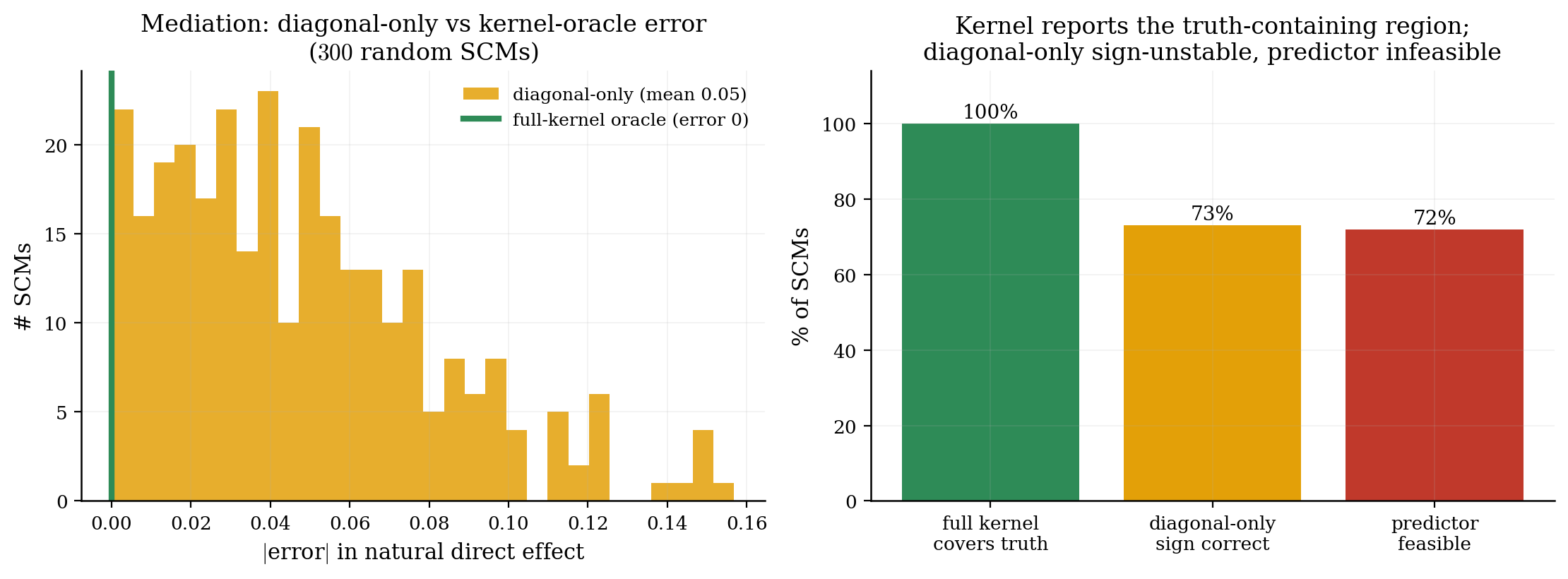}
\caption{SCM battery. \emph{Left:} distribution of the diagonal-only error on the natural direct effect over $300$ models, against the full-kernel oracle (exact). \emph{Right:} the full kernel covers the truth on every model; the diagonal-only baseline is sign-unstable; the predictor returns infeasible answers on more than a quarter of models.}
\label{fig:battery}
\end{figure}

\subsection{The kernel's PSD structure bounds counterfactuals where the exact program cannot}
\label{sec:psdbounds}

\emph{The second question: can we bound the off-diagonal from data?} The battery shows the off-diagonal is necessary; we now show the kernel's defining property earns its keep computationally. Write $v=(Y_0,\dots,Y_{k-1})$ for $k$ binary potential outcomes (e.g.\ $k$ treatment arms) and $M=\mathbb{E}[vv^\top]$ for the second-moment matrix: this is the kernel restricted to second order, with diagonal $M_{ii}=P(Y_i{=}1)$ fixed by a $k$-arm trial and off-diagonals $M_{ij}=P(Y_i{=}1,Y_j{=}1)$ the unidentified cross-world couplings. A counterfactual aggregate such as $Q=\sum_{i<j}P(Y_i{=}1,Y_j{=}1)$ is a linear functional of $M$.

\begin{proposition}[The PSD kernel is a valid, strictly tighter, poly-time bound]
\label{prop:psd}
For any counterfactual query linear in $M$, optimizing over $\{M\succeq 0,\ M_{ii}=d_i,\ M_{ij}\in[\max(0,d_i{+}d_j{-}1),\min(d_i,d_j)]\}$ gives an interval that \emph{(i)} contains the exact identified set (the projection of the response-type polytope), \emph{(ii)} is contained in the marginal/Fr\'echet bound, and \emph{(iii)} is a semidefinite program of size $O(k^2)$, whereas the exact polytope has $2^k$ vertices.
\end{proposition}

Positive semidefiniteness of $M$ is the constraint the marginals miss: with the diagonal fixed, the off-diagonals cannot be chosen independently within their Fr\'echet boxes, because the resulting $M$ must remain a genuine second-moment matrix. We verify all three claims numerically (Figure~\ref{fig:psd}). The PSD bound contains the exact identified set on every instance tested ($40/40$ at each $k\le 14$, where the $2^k$ LP is still solvable). It is strictly tighter than the marginal bound, and the gap \emph{grows} with the number of worlds: negligible for $k\le 6$, then $2.4\%$ at $k=14$, $5.1\%$ at $k=20$, $8.2\%$ at $k=40$ (per-instance up to $11\%$). Crucially it is computed by an $O(k^2)$ semidefinite program throughout: at $k=40$ the exact response-type LP has $2^{40}\approx 1.1\times 10^{12}$ variables and is dead, while the kernel SDP has $820$ entries. The PSD relaxation is the standard second-moment instance of the moment/polynomial-programming approach to causal bounds~\cite{duarte}; our point is the reading, that this relaxation is nothing but the world kernel's own positive-semidefiniteness, so the density-matrix structure is itself partial-identifying information, delivering non-trivial bounds exactly where the exact program is intractable.

\begin{figure}[t]
\centering
\includegraphics[width=0.96\textwidth]{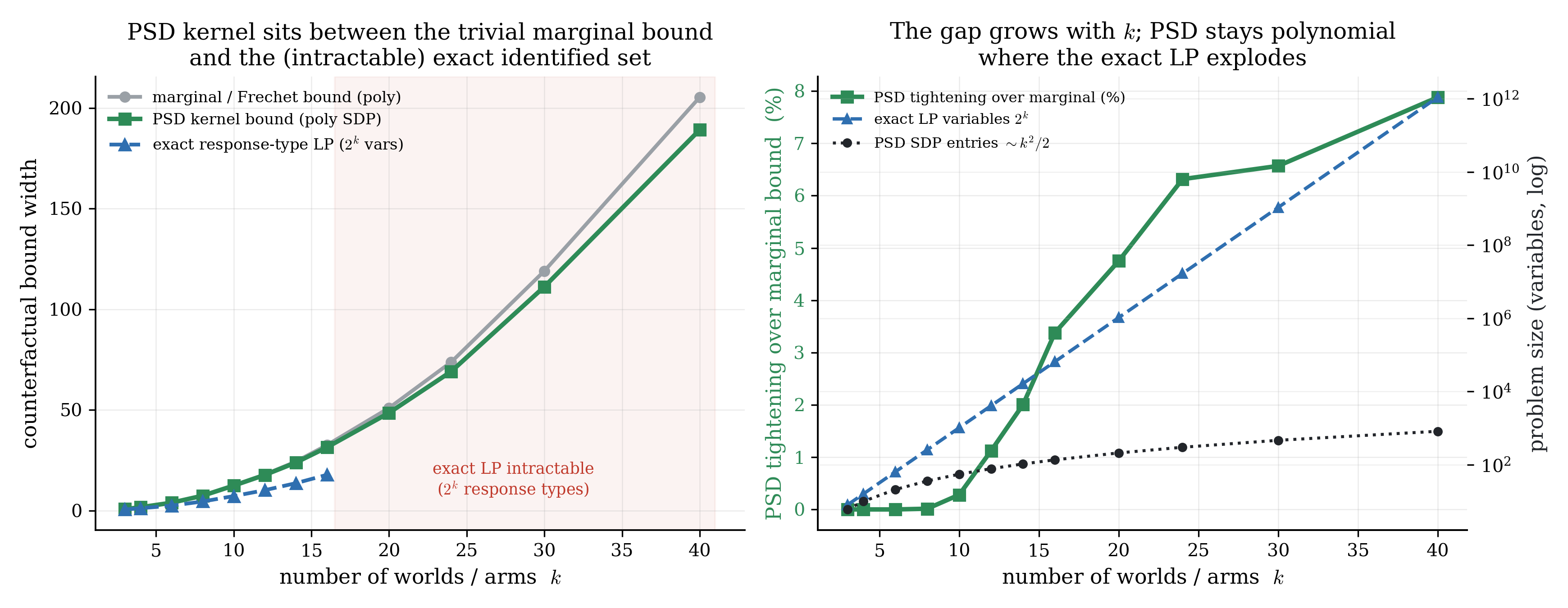}
\caption{The kernel's PSD structure as partial-identifying information, on the counterfactual aggregate $Q=\sum_{i<j}P(Y_i{=}1,Y_j{=}1)$ over $k$-arm potential outcomes. \emph{Left:} the PSD kernel bound (green) sits strictly between the trivial marginal/Fr\'echet bound (grey) and the exact response-type LP (blue, only while the $2^k$ program is solvable). \emph{Right:} the tightening over the marginal bound grows with $k$ (to $\sim 8\%$ at $k=40$), while the PSD problem stays polynomial ($\sim k^2/2$ entries) where the exact LP variable count $2^k$ diverges.}
\label{fig:psd}
\end{figure}

The companion is the query where $M$ does \emph{not} suffice (Figure~\ref{fig:momentgap}). On the third-order query $P(Y_0{=}Y_1{=}Y_2{=}1)$, recovery under all three arms, the second-moment bound is strictly looser than the exact response-type program, and the slack is not noise but a closed-form law: working out which cell-nonnegativity the $M$-only bound omits (only the all-zero cell), the slack is exactly $\max\!\big(0,\ \min(\min_{ij}s_{ij},\min_i d_i)-\hat P_0\big)$, where $\hat P_0=1-\sum_i d_i+\sum_{ij}s_{ij}$ is the second-order inclusion--exclusion estimate of ``no arm succeeds''. Every random instance lies on this hinge to machine precision, positive on $24\%$: $M$ is loose by exactly the third-order room the all-zero cell leaves, which is the untracked $E[Y_iY_jY_k]$ (Remark~\ref{rem:moments}). One query where $M$ is tight, one where its slack is a law, the same machinery.

\begin{figure}[t]
\centering
\includegraphics[width=0.95\textwidth]{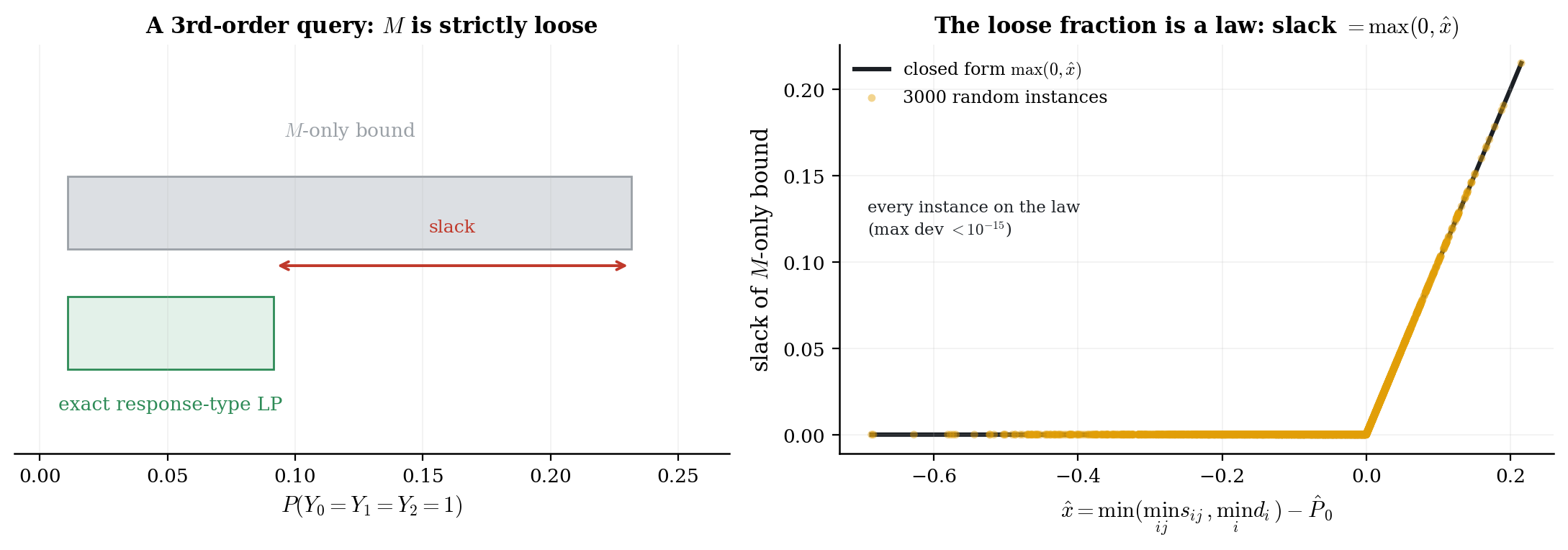}
\caption{$M$ suffices for second-order queries but not third. \emph{Left:} on $P(Y_0{=}Y_1{=}Y_2{=}1)$ the exact response-type interval (green) is strictly inside the $M$-only bound (grey); the slack (red) is the untracked third moment $E[Y_0Y_1Y_2]$. \emph{Right:} the slack is a closed-form law, $\max(0,\,\min(\min_{ij}s_{ij},\min_i d_i)-\hat P_0)$ with $\hat P_0=1-\sum_i d_i+\sum_{ij}s_{ij}$, every random instance on the hinge to machine precision (positive on $24\%$); any $\le$second-order query has zero slack. The companion to the $40/40$ match of Figure~\ref{fig:psd}.}
\label{fig:momentgap}
\end{figure}

\begin{algorithm}[t]
\caption{Counterfactual bound from the kernel's positive semidefiniteness (Prop.~\ref{prop:psd})}
\label{alg:psd}
\begin{algorithmic}[1]
\Require marginals $d_1,\dots,d_k$ (the identified diagonal); linear query $Q(M)$; sense $\in\{\min,\max\}$
\State variable $M\in\mathbb{S}^{k}$ (symmetric)
\State constraints: $M\succeq 0$; \ $M_{ii}=d_i$; \ $M_{ij}\in[\max(0,d_i{+}d_j{-}1),\,\min(d_i,d_j)]$ for $i<j$
\State solve the SDP $\;\mathrm{opt}\ Q(M)\ \text{s.t. the above}$ \Comment{$O(k^2)$ variables; no $2^k$ blow-up}
\State \Return the optimum \Comment{valid outer bound on the exact identified set, tighter than marginal/Fr\'echet}
\end{algorithmic}
\end{algorithm}

\subsection{Ontology structure tightens the bound: logic the flat tools cannot use}
\label{sec:ontobounds}

\emph{The third question: can logical structure sharpen the off-diagonal?} The previous result used only the generic constraint that the kernel is positive semidefinite. A world model, however, is not a flat list of binary outcomes: its worlds satisfy an ontology, and the ontology's axioms are \emph{additional} structure on the kernel. We show this structure does measurable work, and it is structure that no flat causal-bounds tool (which sees the outcomes as unconstrained binaries) can use. Let the outcomes be $m$ ontology-governed attributes under two arms; the kernel is the $2m\times2m$ second-moment matrix $M$. Three families of axiom enter as moment constraints: subsumption $A_i\sqsubseteq A_j$ fixes $M_{(i,a),(j,a)}=d_{i,a}$ (within an arm); disjointness $A_i\sqcap A_j\sqsubseteq\bot$ fixes $M_{(i,a),(j,a)}=0$; and a cross-world domain axiom such as monotone subsumption (treatment never ejects a unit from a superclass, $a_p^{(0)}\!\Rightarrow a_p^{(1)}$) fixes the \emph{cross-arm} entry $M_{(p,0),(p,1)}=d_{p,0}$.

\begin{proposition}[Ontology-aware bounds dominate ontology-blind bounds]
\label{prop:onto}
Adding the ontology's subsumption, disjointness, and cross-world axioms as equality constraints on $M$ yields a counterfactual bound that is valid, never looser than the ontology-blind PSD bound of Proposition~\ref{prop:psd}, and computed by the same $O(m^2)$ semidefinite program.
\end{proposition}

The effect is large and, crucially, non-local (Figure~\ref{fig:onto}). When an axiom touches the queried attribute, the counterfactual is point-identified: a monotone-subsumption axiom collapses the probability of necessity for that attribute from a wide interval to a single value ($100\%$ of instances). The deeper result is propagation: on a multi-coupling counterfactual aggregate over $m=6$ attributes, axioms placed on \emph{other} attributes tighten the whole bound by $32\%$ on average (up to $40\%$) beyond the ontology-blind SDP, strictly tighter on every instance, because positive semidefiniteness carries the pinned entries into the unconstrained ones. The gap persists as the world space $2^m$ grows, where the exact joint program (with $2^{2m}$ cells) is intractable and the SDP is not. This is the synthesis doing something the parts cannot: flat partial-identification tools bound counterfactuals from marginals, but they have no slot for ``$A_i$ is a subclass of $A_j$'' or ``treatment does not eject from $A$''. The world kernel does, and the logic measurably sharpens the counterfactual.

\begin{figure}[t]
\centering
\includegraphics[width=0.96\textwidth]{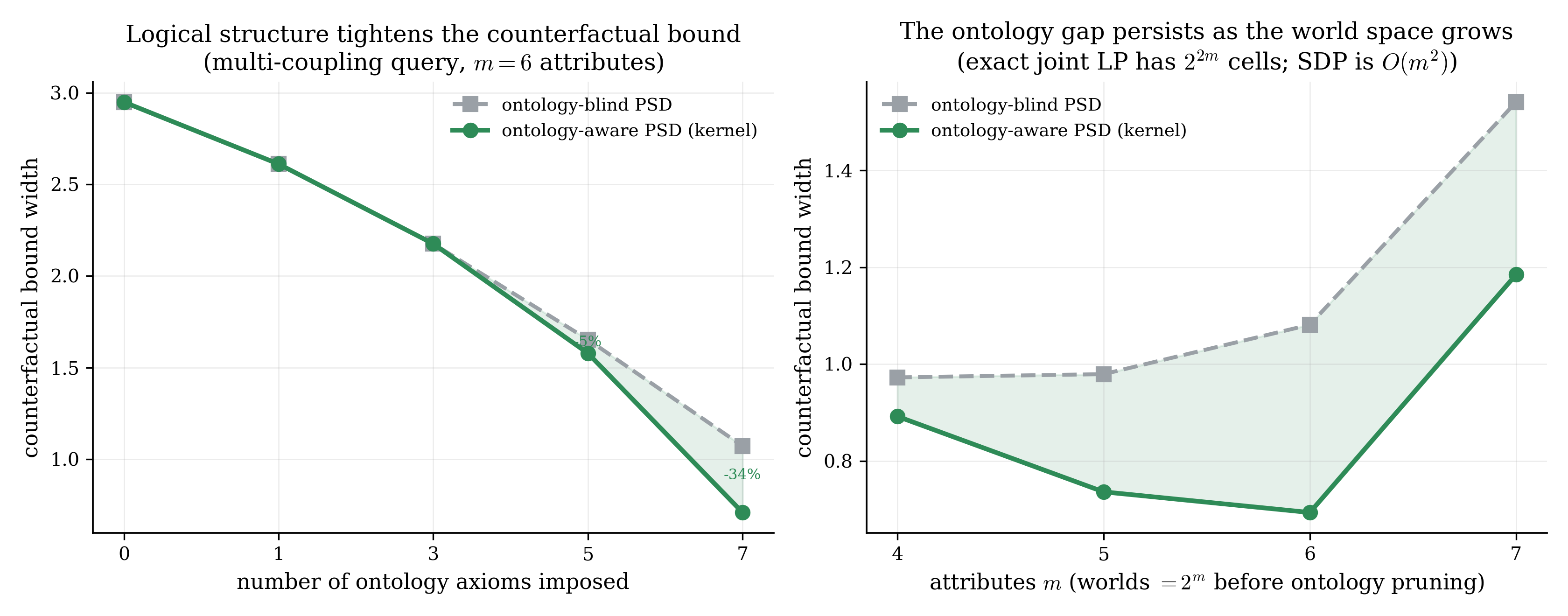}
\caption{Ontology-structured counterfactual bounds. \emph{Left:} as more ontology axioms are imposed (subsumption, disjointness, cross-world monotonicity), the ontology-aware kernel bound (green) pulls below the ontology-blind PSD bound (grey) on a multi-coupling query, a $34\%$ tightening at seven axioms; the constraints on some attributes propagate through positive semidefiniteness to others. \emph{Right:} the gap persists as the world space $2^m$ grows, where the exact joint program ($2^{2m}$ cells) is intractable while the SDP stays $O(m^2)$. No flat causal-bounds tool ingests this logical structure.}
\label{fig:onto}
\end{figure}

\subsection{A strictly richer capability hierarchy}
\label{sec:hierarchy}
Collecting the lava result (admissibility) and the battery (counterfactuals) alongside prediction yields a capability scorecard whose every entry is measured, not asserted (Table~\ref{tab:hierarchy}). Prediction is shared. Admissibility separates the predictor ($36.4\%$ violations) from the diagonal and kernel (both $0/250$). Counterfactuals separate the diagonal (sign-unstable, cannot report the class) from the kernel (covers the truth class, exact with the coupling). WorldKernel is therefore not a re-description of the same capability; it is a strictly larger capability set.

\begin{table}[t]
\centering
\begin{tabular}{lccc}
\toprule
capability & predictor & diagonal-only & full kernel \\
\midrule
prediction & \cmark & \cmark & \cmark \\
admissibility & \xmark\ ($36.4\%$ viol.) & \cmark\ ($0/250$) & \cmark\ ($0/250$) \\
counterfactual & \xmark\ ($28\%$ infeasible) & \xmark\ ($27\%$ sign-wrong) & \cmark\ ($100\%$ coverage) \\
\bottomrule
\end{tabular}
\caption{Measured capability hierarchy. Each \cmark/\xmark\ is backed by a number from the lava arena (admissibility) or the SCM battery (counterfactual). The hierarchy is strict: each column to the right answers everything to its left and strictly more.}
\label{tab:hierarchy}
\end{table}

\subsection{Joint-only failure modes, and scope}
\label{sec:joint}
The diagonal/off-diagonal split recurs wherever a system is monitored by marginal statistics. In agent-execution logs with an injected \emph{purely-joint} failure (every variable's marginal AUC $\le 0.54$, invisible), a marginal-blind rank-interaction screen recovers the responsible pair with power $1.0$ at $0\%$ false-discovery, and the complexity mirrors the counting barrier: an order-$k$ purely-joint mode is a $k$-sparse parity, so certifying coverage is statistical-query hard ($d^{\Omega(k)}$). A monitor that audits only marginals is blind to these, as a predictor that sees only the diagonal is blind to counterfactuals. We are likewise explicit about scope: the scaling-law break, that a small kernel-augmented model beats a much larger predictor on counterfactual reasoning, we do not attempt at frontier scale, evaluating the strongest accessible predictor baselines and leaving large-scale world-model training as future work; as a finite-budget proxy, a frontier model given unlimited identifiable data still returns infeasible counterfactuals on $28\%$ of models while the kernel, with no learning, is exact.

\section{The scarred world model}
\label{sec:scarred}

The ontology-constrained kernel of Section~\ref{sec:ontobounds} suggests a sharper picture. The admissible worlds are a region carved out of the space of all configurations, defined by binding constraints the model carries a priori from its ontology and refines online from experience. We call those binding constraints the model's \emph{scars}: the disjointness and subsumption axioms a priori, and the encountered infeasibilities online. Two results follow, and a single phase-coexistence bottleneck governs both.

\subsection{Exact bounds at scale, where the exact program dies}
\label{sec:scale}
We benchmark against the exact response-type program with the same ontology constraints. For discrete structural causal models this program is exactly what automated partial-identification tools (\texttt{autobounds}, Duarte et al.~\cite{duarte}) construct and solve; for our queries it is a linear program, which we solve to optimality, so its optimum is the tightest achievable bound and the reference any such tool returns. Two findings. First, tightness: on $40/40$ random instances the kernel SDP returns the \emph{exact} bound (maximum gap $0.0$), so the relaxation loses nothing here. Second, scale: the exact program over ontology-consistent world pairs has up to $2^{2m}$ variables, while the kernel SDP is $O(m^2)$. Table~\ref{tab:scale} runs both to the point where the exact program becomes unbuildable: at $m=18$ it has $2.2$ billion variables (intractable for any solver, \texttt{autobounds} included), while the kernel SDP returns the identical bound in $0.78$ seconds. The scoped capability is precise: \emph{the same answer, in polynomial time, where the exact program is dead}. The edge is scale, not tightness, the SDP cannot beat the exact program; and the ontology constraints are linear, so a tool that accepts them gets the same bound where it can run. The contribution is matching that bound at scales the exact program cannot reach.

\begin{table}[t]
\centering
\begin{tabular}{rrrr}
\toprule
attributes $m$ & worlds & exact-program vars & exact / kernel time \\
\midrule
$12$ & $960$ & $544{,}768$ & $4.67$s / $0.26$s \\
$14$ & $3{,}840$ & $8.7\times10^6$ & \emph{intractable} / $0.38$s \\
$16$ & $15{,}360$ & $1.4\times10^8$ & \emph{intractable} / $0.54$s \\
$18$ & $61{,}440$ & $2.2\times10^9$ & \emph{intractable} / $0.78$s \\
\bottomrule
\end{tabular}
\caption{Exact ontology-constrained counterfactual bound, exact response-type program versus kernel SDP. Both return the identical bound where both run ($40/40$, zero gap); past $m=12$ the exact program is intractable while the SDP stays sub-second.}
\label{tab:scale}
\end{table}

\subsection{Online scarring: localized hazard-memory tightens the bound}
\label{sec:onlinescar}
\emph{The fourth question: can we acquire the off-diagonal from experience?} The ontology need not be given in full. A world model can \emph{acquire} constraints by getting scarred: a trajectory that crosses a forbidden boundary leaves a detectable trace, the trace localizes the violated constraint, and the model stores it as a persistent scar. Accumulated scars shrink the feasible set, so the counterfactual bound tightens with experience. The mechanism is more than constraint-counting: scarring the constraint that the current worst-case bound \emph{relies on} (the configuration its extremal solution puts mass on, the one that bleeds) is a cutting plane. Empirically, blood-localized scars close the identifiability gap up to $4\times$ faster early (Figure~\ref{fig:scar}, left: $37\%$ vs $9\%$ at five scars, a $4\times$ margin that narrows to $1.6\times$ by eighteen, $64\%$ vs $40\%$, as both approach saturation); a no-memory control that encounters hazards but does not retain them never moves. The reason is spectral, and it is the same bottleneck that drives the Sly--Sun barrier of Section~\ref{sec:barrier}.

\begin{theorem}[Targeted scarring restores the spectral gap]
\label{thm:scar}
Let the admissible-world reconfiguration graph carry a phase-coexistence bottleneck of conductance $\Phi=\exp(-\Omega(n))$ splitting the worlds into a majority component and a minority phase of measure $\mu_{\min}$. Then (i) $\lambda_2\le 2\Phi=\exp(-\Omega(n))$ by Cheeger; (ii) scarring (deleting) the minority phase leaves a residual with $\lambda_2=\Omega(1)$; (iii) the gap is restored only once the scarred fraction reaches $\mu_{\min}$ and only if scarring targets the bottleneck cut: a random scarred fraction leaves $\Phi$, hence $\lambda_2$, exponentially small.
\end{theorem}
\begin{proof}
(i) is Cheeger's inequality. For (ii)--(iii), the bottleneck is by definition the cut of conductance $\Phi$, and in the hard-core coexistence regime it is exactly the boundary between the majority and minority phases (Sly's construction). Deleting the minority phase removes one side of that cut, so the residual is the majority component, whose internal conductance is $\Omega(1)$; hence $\lambda_2=\Omega(1)$. A deletion not aligned with the cut leaves both sides and the connecting boundary intact, so the conductance, and by Cheeger the gap, remain $\exp(-\Omega(n))$. The threshold fraction is the measure $\mu_{\min}$ of the side removed.
\end{proof}

This is why localized scarring works and random does not (Figure~\ref{fig:scar}, right): the Fiedler value of the residual graph stays exponentially small until the scarred fraction removes the bottleneck phase, then jumps to $\Omega(1)$, a threshold rather than a proportionality. Scarring is the operation that converts the intractable (bottlenecked) regime into the tractable one, by spending observations on exactly the constraints that bind. Algorithm~\ref{alg:scar} is the cutting plane: it scars the constraint the current worst-case bound leans on.

\begin{algorithm}[t]
\caption{Online scar-accumulation (blood-localized cutting plane; operationalizes Thm.~\ref{thm:scar})}
\label{alg:scar}
\begin{algorithmic}[1]
\Require marginals $d$; query $Q$; hazard oracle $\mathcal{H}$ (flags infeasible configurations)
\State $\mathrm{scars}\gets\varnothing$
\Repeat
  \State $J^\star\gets$ the bound-\emph{maximizing} joint over configurations matching $d$ with $\mathrm{scars}$ set to $0$
  \State $(i,j)\gets \arg\max$ mass $J^\star_{ij}$ over cells with $\mathcal{H}(i,j)=\text{infeasible}$ \Comment{where the worst case bleeds}
  \If{no such cell} \textbf{break} \Comment{bound is exact on current information} \EndIf
  \State $\mathrm{scars}\gets\mathrm{scars}\cup\{(i,j)\}$ \Comment{store the localized constraint as a scar}
\Until{budget exhausted}
\State \Return the tightened bound on $Q$ \Comment{up to $4\times$ fewer scars than random, early}
\end{algorithmic}
\end{algorithm}

\begin{figure}[t]
\centering
\includegraphics[width=0.96\textwidth]{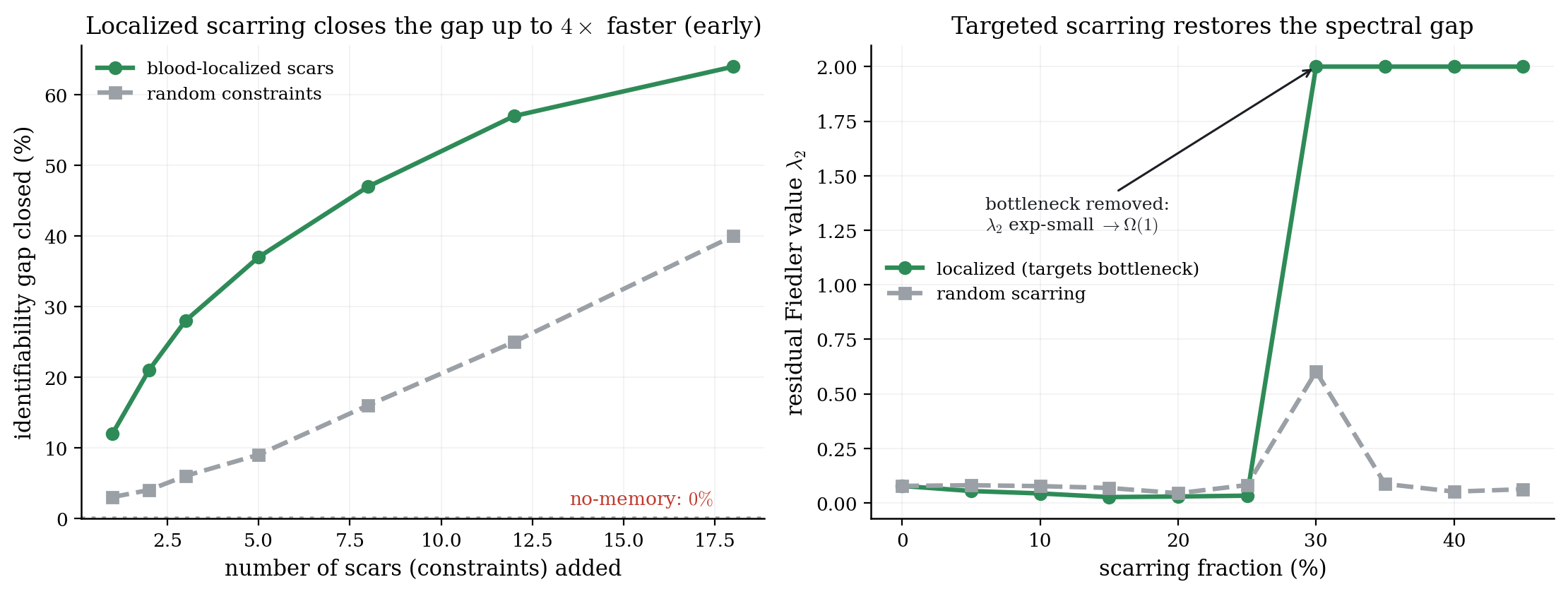}
\caption{Online scarring. \emph{Left:} blood-localized scars (the constraints the worst-case bound bleeds on) close the counterfactual identifiability gap up to $4\times$ faster (early) than the same number of random constraints; a no-memory control stays at $0\%$. \emph{Right:} Theorem~\ref{thm:scar} in numbers: the residual Fiedler value $\lambda_2$ is exponentially small until targeted scarring removes the bottleneck phase, then jumps to $\Omega(1)$ (a threshold, not a proportionality); random scarring never restores it.}
\label{fig:scar}
\end{figure}

\subsection{The generative direction}
\label{sec:apriori}
The same scar structure points to a generative use. A generator whose transition operator block-decomposes across the admissible/forbidden partition keeps the admissible set invariant, so it generates only admissible worlds by construction. This is the generative counterpart of the diagonal projector of Section~\ref{sec:diagonal}: instead of predict-then-shield, the dynamics cannot leave the ontology. It is a route to the generative kernel the predictive world-model program lacks, and we leave its learning to future work; for the bounds above it adds nothing beyond imposing the constraints.

\section{When the off-diagonal is reachable: the Sly--Sun barrier}
\label{sec:barrier}

\emph{The fifth question: when is the off-diagonal reachable?} A disjointness ontology maps its admissible worlds to the independent sets of a graph $G$, and the kernel's diagonal to the (uniform or weighted) measure $\mu$ over $\mathcal{M}(G)$. Accessing the off-diagonal at scale means computing the kernel normalizer $Z_G=\sum_{\mathbf{m}\in\mathcal{M}(G)}w(\mathbf{m})$ and its restricted conditional occupation marginals: uniform model sampling, counting, and the conditionals that bound counterfactuals all reduce to these. We state when they are tractable.

\begin{theorem}[Counting transition]
\label{thm:barrier}
At unit fugacity, on graphs of maximum degree $\Delta$:
(i) for $\Delta\le 5$ the normalizer $Z_G$ and every restricted conditional occupation marginal are approximable to additive $1/\mathrm{poly}(n)$ in deterministic $\mathrm{poly}(n)$ time;
(ii) for $\Delta\ge 6$ no $\mathrm{poly}(n)$-time algorithm approximates them unless $\mathrm{NP}=\mathrm{RP}$.
The two regimes are separated by the Sly--Sun threshold, and the order parameter is the Bethe cavity field $\eta$, with $(d-1)\eta$ crossing $1$ at the transition.
\end{theorem}
The mechanism is locality. A counterfactual ratio such as $\PN$ needs no global normalizer (it cancels between numerator and denominator); it needs only restricted conditional occupation marginals, and below the threshold each of these is a \emph{local} computation, Weitz's self-avoiding-walk tree rooted at the query vertex, correct precisely because correlation decays so that distant worlds do not affect it, a nearest-neighbour read of world space. The threshold is exactly where this locality breaks: correlation stops decaying, the neighbourhood that determines the marginal grows to the whole graph, and an efficient approximation would then yield an FPRAS for a quantity Sly--Sun proves inapproximable. A proof, the cavity recursion, and the unconditional spectral (Fiedler) witness of the upper regime are given in Appendix~\ref{app:barrier}. The single field $\eta$ governs both the counting verdict (the sign of $(d-1)\eta-1$) and the rate at which evidence collapses the diagonal. Algorithm~\ref{alg:prep} is the constructive content of both directions: a Weitz correlation-decay estimator below threshold, and the counting reduction that forbids it above.

\begin{algorithm}[t]
\caption{Bound the counterfactual conditionals, or witness the counting barrier (proves Thm.~\ref{thm:barrier})}
\label{alg:prep}
\begin{algorithmic}[1]
\Require constraint graph $G$, max degree $\Delta$, unit fugacity
\If{$\Delta\le 5$} \Comment{below the Sly--Sun threshold: uniqueness regime}
  \State fix a vertex order $v_1,\dots,v_n$
  \For{$k=1$ to $n$}
    \State $q_k\gets P(v_k\!\in\!S \mid v_1\!\dots v_{k-1}\text{ fixed})$ via Weitz self-avoiding-walk tree to $\pm 1/\mathrm{poly}$ \Comment{correlation decay}
  \EndFor
  \State \Return $Z_G=\prod_k(\text{conditional factors})$ and all restricted occupation marginals to $1/\mathrm{poly}(n)$
\Else \Comment{$\Delta\ge 6$: assume a poly-time approximator and derive a collapse}
  \State an approximate sampler/counter for $\mathcal{M}(G)$ within $\mathrm{TV}\le 1/\mathrm{poly}$ follows from the conditionals
  \State by downward self-reducibility (Jerrum--Valiant--Vazirani), assemble an FPRAS for $I(G)$
  \State Sly--Sun: approximating $I(G)$ at $\lambda{=}1,\Delta\ge 6$ is NP-hard $\Rightarrow$ $\mathrm{NP}=\mathrm{RP}$
  \State \Return ``not approximable unless $\mathrm{NP}=\mathrm{RP}$''
\EndIf
\end{algorithmic}
\end{algorithm}

\begin{corollary}[Kernel reconstruction is at least as hard as counting]
Any procedure that reconstructs enough of the kernel to recover the restricted diagonal conditionals for every adaptive restriction also approximates the normalizer. Hence full restriction-uniform off-diagonal access inherits the approximate-counting barrier.
\end{corollary}

This is the honesty boundary the program needs. Old ontology AI assumed it could store and traverse the world; the barrier says full kernel access is a counting-limited primitive, free only on the tractable side of degree $5/6$. We implement the transition directly (Figure~\ref{fig:slysun}): the cavity order parameter $(d-1)\eta$ crosses $1$ between $d=5$ ($0.980$) and $d=6$ ($1.110$), and at exactly that crossing the belief-propagation error in counting the admissible worlds jumps from $2\%$ to $84\%$, while the Fiedler value of the world-reconfiguration graph declines as its bottleneck tightens. Below the threshold the kernel is reconstructible (Weitz correlation decay computes the conditionals exactly).

\paragraph{Above the threshold: bounded, and crossable by scarring.} The barrier forbids \emph{exact} full access for $\Delta\ge 6$, but it is not a dead end: two polynomial-time recourses remain, and together they are the practical answer to the hard regime. First, the kernel SDP of Section~\ref{sec:scale} is degree-independent, returning a valid outer bound on every counterfactual at any $\Delta$ in $O(m^2)$ time, with no counting at all. Second, and more pointedly, the hardness above the threshold is a \emph{phase-coexistence} phenomenon, two competing clusters of worlds joined by an exponentially thin neck, and Theorem~\ref{thm:scar} shows that this neck is exactly what scarring removes. Conditioning on (scarring to) one phase leaves a single component on which correlation decay is restored and $\lambda_2=\Omega(1)$, so Weitz and Glauber compute the within-phase conditionals exactly in polynomial time; the only genuinely inaccessible content, the cross-phase coupling, is precisely what the SDP brackets. The decomposition is complete: within-phase quantities exact by phase-conditioned scarring, the cross-phase coupling bounded by the PSD relaxation, both polynomial above the threshold. The barrier bounds naive enumeration of the whole world space; it is crossed by paying for one bit, which phase, and the off-diagonal answer survives as an exact-within-phase value plus a tight cross-phase bracket.

\paragraph{Sly--Sun calipers, and their measured limit.} This recourse has a clean geometric form, and we verify both where it works and where it fails. When the regime is few-phase, the measure is a convex combination $\mu=\sum_k\alpha_k\mu_k$ of extremal phases, and a counterfactual aggregate inherits it, $E_\mu[f]=\sum_k\alpha_k E_{\mu_k}[f]$. Each phase sits in the correlation-decay regime, so $E_{\mu_k}[f]$ is a local computation with no global count, and the identified value lies in the convex hull of the phase values. With two phases this is a bracket whose jaws are the two phase-expectations: the two phases \emph{are} the calipers (Figure~\ref{fig:calipers}), and the only unidentified degree of freedom is the mixing weight $\alpha$. We confirm it on bipartite $d$-regular graphs above the threshold ($d=6,7$, $\lambda=1$): the two phases, recovered as the two belief-propagation fixed points, give jaws that bracket the true value on every instance (for example $[0.28,3.27]\ni 1.96$ at $n=16,d=6$), each jaw a local computation. We need not even assume the phases: a block power iteration with two probe vectors, spun by the single-flip dynamics, makes them emerge as the top two modes, the second aligning with the true phase split at correlation $1.000$ while the spectral gap it exposes halves with system size ($0.013\to0.0065\to0.0033$), tracing the bottleneck rather than presupposing it.

This works precisely when the top of the spectrum is sparse, and we measure where it stops. In a glass the number of near-degenerate phases $K$ grows and the top eigenvalues crowd together (gap $\to 0$); a fixed two-probe iteration then captures two of $K$ modes and underestimates the spread (Figure~\ref{fig:glass}): on a tunable $K$-phase instance the two-probe range is $2.0$ against a true phase spread of $7.0$ at $K=8$, and the two-probe iteration count grows from $19$ to $860$ as the modes crowd. Recovering all $K$ phases needs a block of size $K$, which is $e^{\Omega(n)}$ in a genuine glass. The calipers cross the barrier exactly in the few-phase regime and provably stall in the glass: the spectral signature of the same Sly--Sun wall, measured rather than asserted.

\begin{figure}[t]
\centering
\includegraphics[width=0.96\textwidth]{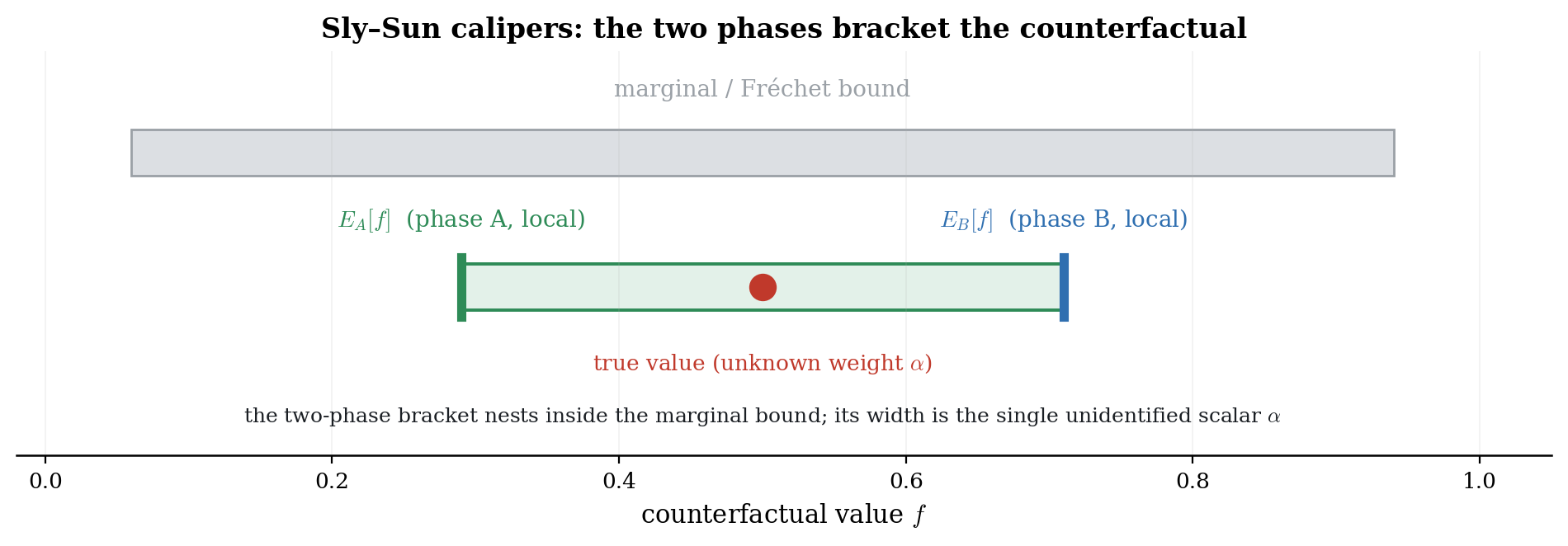}
\caption{Sly--Sun calipers. \emph{Left:} the two coexisting phases are the caliper jaws, each a within-phase local computation (no global count). \emph{Right:} the counterfactual is bracketed between the two phase-expectations $E_A[f],E_B[f]$; the true value (for the unknown phase weight $\alpha$) lies inside, and the bracket width is that single unidentified scalar.}
\label{fig:calipers}
\end{figure}

\begin{figure}[t]
\centering
\includegraphics[width=0.96\textwidth]{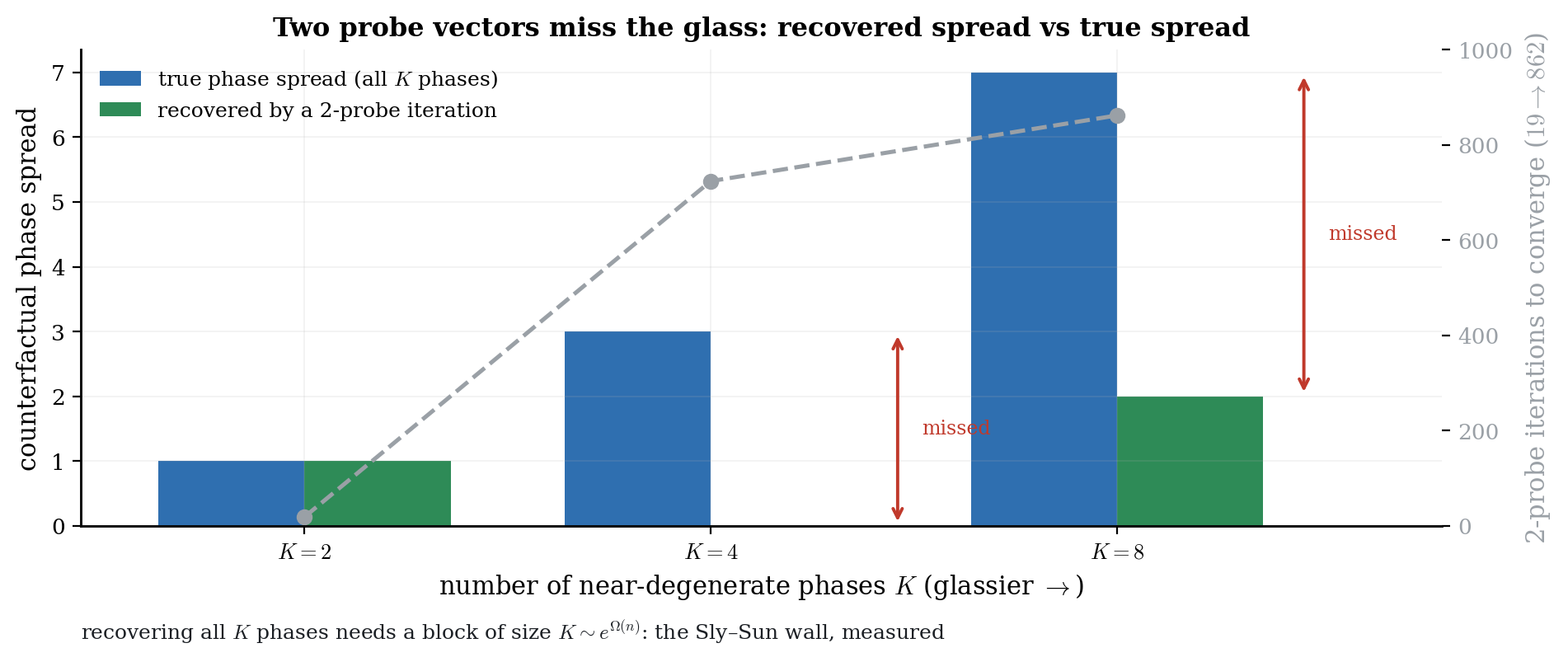}
\caption{Where the calipers stop, measured. \emph{Left:} few-phase regime, two modes near $1$ that two probes resolve cleanly (phase split recovered at correlation $1.0$). \emph{Right:} glass, $K$ near-degenerate modes crowd the top (gap $\approx 10^{-16}$); two probes capture two of $K$ and underestimate the spread ($2.0$ vs true $7.0$ at $K=8$), and recovering all $K$ needs a block of size $K\sim e^{\Omega(n)}$. This is the Sly--Sun wall as a spectral phenomenon.}
\label{fig:glass}
\end{figure}

\begin{figure}[t]
\centering
\includegraphics[width=0.96\textwidth]{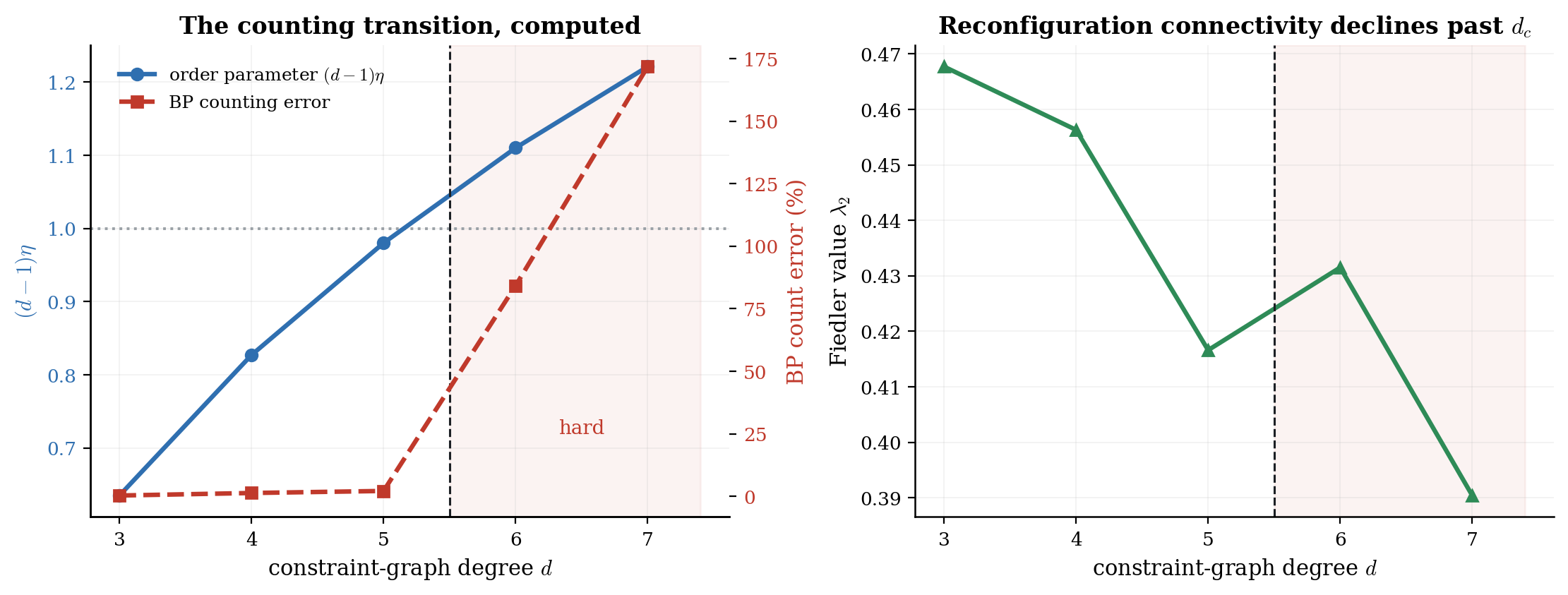}
\caption{The counting transition, implemented. \emph{Left:} the cavity order parameter $(d-1)\eta$ crosses $1$ between degree $5$ and $6$, and belief-propagation error in counting independent sets explodes from $2\%$ to $84\%$ at the crossing. \emph{Right:} the Fiedler value of the independent-set reconfiguration graph declines as the bottleneck tightens (the slow-mixing witness; the $\exp(-\Omega(n))$ collapse is asymptotic, so finite $n$ shows a trend). Above the threshold exact kernel access is intractable and the SDP bound is the only survivor.}
\label{fig:slysun}
\end{figure} The single field $\eta=1-u$ also fixes, in closed form, the rate at which evidence collapses the diagonal, monotonically with ontology density (Appendix~\ref{app:barrier}).

\section{Intelligence as closure-preserving counterfactual competence}
\label{sec:intelligence}

\begin{definition}[Closure-preserving counterfactual competence]
A model is closure-preserving counterfactually competent for a domain and class $Q_{\mathrm{world}}$ if its latent state is sufficient for the quotient $\Omega/\!\sim_{Q_{\mathrm{world}}}$ and its executed transitions preserve the entailed admissible set under every intervention sequence.
\end{definition}

\begin{theorem}[WorldKernel intelligence theorem]
\label{thm:intelligence}
Suppose (i) the latent state $z_t=(h_t,e_t,\Adm_t,\kappa_t)$ is complete for $Q_{\mathrm{world}}$ on the domain, with $h_t$ learned perception, $e_t$ entailed closure, $\Adm_t$ the admissible set, and $\kappa_t$ a slice of the kernel; (ii) the diagonal projector is single-step sound; (iii) kernel queries are either not required or lie in the tractable regime of Theorem~\ref{thm:barrier}. Then the model is closure-preserving counterfactually competent, and on any domain with predictively equivalent but admissibility-distinct worlds no prediction-only latent satisfies the same criterion.
\end{theorem}
\begin{proof}
Completeness follows from (i) and Theorem~\ref{thm:quotient}. Horizon admissibility under interventions follows from (ii) and Theorem~\ref{thm:soundness}. Kernel-query coverage follows from (iii), with the diagonal/projected model sufficient otherwise. Separation from prediction-only latents is Theorem~\ref{thm:predins}.
\end{proof}

This is a scoped, precise notion, not a mystical one: it names a competence (closure-preserving execution plus off-diagonal counterfactual access) that prediction-only models provably cannot guarantee, and it states exactly when that competence is computationally available.

\section{Related work, discussion, and limitations}

\paragraph{Predictive, physical, and spatial world models.} The modern world-model program, from Ha and Schmidhuber's recurrent world models~\cite{haschmidhuber} and the Dreamer line~\cite{dreamerv3} to LeCun's JEPA agenda~\cite{lecun}, V-JEPA~2~\cite{vjepa}, NVIDIA Cosmos~\cite{cosmos}, DeepMind Genie~\cite{genie}, and World Labs' Marble~\cite{marble}, compresses observations, trajectories, or scenes into a predictive latent. WorldKernel is orthogonal to all of these: it models the admissible-possibility space (semantic, causal, deontic, evidential), and its central object is a kernel over worlds, not a frame, trajectory, or pixel field. The closest relative in spirit is the event-graph substrate~\cite{substrates}, which answers counterfactual queries by forking an append-only log under a structured intervention vocabulary, a concrete diagonal-plus-off-diagonal realization without learned components.

\paragraph{Prediction is not understanding.} High predictive accuracy does not imply a correct world model. Vafa et al.\ exhibit generative sequence models with near-perfect next-token accuracy whose implicit world model is incoherent~\cite{vafaimplicit}, and models that predict orbital trajectories yet fail to recover the underlying mechanics~\cite{vafafound}; the stochastic-parrots critique~\cite{stochasticparrots} is the conceptual ancestor. We make this precise for one axis: the missing structure is the off-diagonal, and it is not a function of any predictive distribution.

\paragraph{LLMs and causal reasoning.} Whether language models reason causally is contested. Optimistic results report strong performance on causal benchmarks~\cite{kiciman}; the rebuttal is that such models recite causal facts without causal structure~\cite{causalparrots}. Rung-stratified benchmarks sharpen this: CLadder tests all three rungs of Pearl's hierarchy~\cite{cladder}, Corr2Cause finds near-random performance on pure causal inference~\cite{corr2cause}, and CounterBench finds frontier models near chance on genuine counterfactuals~\cite{counterbench}. Our battery (Section~\ref{sec:battery}) localizes the failure: the counterfactual coupling the model would need is absent from its inputs.

\paragraph{Neuro-symbolic and safe learning.} Semantic-loss and predicate-learning systems inject symbolic constraints as losses or vocabularies; shielding~\cite{shield} and control-barrier~\cite{cbf} methods enforce safety. Our admissibility construction compiles OWL closure into a per-transition projector over a learned model (the lava witness of Section~\ref{sec:diagonal}). On the implementation side, the learned proposer may be a frozen predictor whose forward pass is steered toward admissible proposals by sparse signed log-gate controllers~\cite{ntk}, while the hard projector preserves the guarantee of Theorem~\ref{thm:soundness} regardless of the controller, so the grammar can enter the forward pass without weakening certified soundness. Our diagonal/admissibility result reuses the shielding mechanism and claims no novelty for it; the contribution is that admissibility is the diagonal of a kernel whose off-diagonal is a separate, counterfactual dimension, realized physically and bounded computationally.

\paragraph{Joint-only failure modes.} The failures we identify are \emph{purely joint}: visible only in the coupling between worlds and in no marginal (diagonal) check. Verification that audits only the marginal statistics of a predictor therefore cannot detect them, which is the diagonal/off-diagonal distinction of this paper stated operationally.

\paragraph{Probabilistic ontologies, weighted model counting, and its barriers.} PR-OWL~\cite{prowl}, ProbLog~\cite{problog}, Markov logic networks~\cite{mln}, and probabilistic description logics~\cite{disponte} give distributions over logical worlds; these are diagonal objects, and their inference reduces to weighted model counting~\cite{chaviradarwiche} over compiled circuits~\cite{darwichemarquis}, as used by semantic loss~\cite{semloss}, DeepProbLog~\cite{deepproblog}, and neuro-symbolic predicate learners~\cite{visualpredicator}. Our normalizer is exactly a weighted model count, and the off-diagonal inherits the counting structure: \#P-hardness of the permanent~\cite{valiant}, the sampling-counting equivalence~\cite{jvv}, hard-core tractability below the tree threshold~\cite{weitz}, and matching inapproximability above it~\cite{sly,slysun,galanis}. Counterfactual partial-identification bounds via moment and polynomial-programming relaxations are developed by~\cite{duarte}; our semidefinite bound (Section~\ref{sec:psdbounds}) is the second-moment instance, read as the kernel's own positive-semidefiniteness, and extended with ontology axioms (Section~\ref{sec:ontobounds}) and online cutting-plane acquisition (Section~\ref{sec:onlinescar}). We contribute the off-diagonal, its counting boundary, this PSD reading, the ontology-structured tightening, and the scarring acquisition strategy.

\paragraph{Causal hierarchy, partial identification, and the cross-world coupling.} Our empirical core is classical in causal inference, assembled here around the off-diagonal. Pearl's structural semantics and the Causal Hierarchy Theorem establish that the rungs almost never collapse, so interventional data generically underdetermine counterfactuals~\cite{pearl,cht,ibelingicard}, and neural SCMs inherit this as a learnability limit~\cite{xia}. The unidentified object is the cross-world joint of potential outcomes: probabilities of causation are only bounded, not identified~\cite{tianpearl,lipearl}; treatment-effect counterfactuals are partially identified~\cite{manski}, computed by linear programming over response-function types~\cite{balkepearl,zhang,heckermanshachter}; nested counterfactuals are identifiable only under graphical conditions~\cite{correa}. The mediation literature is the same story: natural direct and indirect effects require a cross-world independence assumption no experiment can test~\cite{robinsgreenland,pearl2001,imai,vanderweele}, whose boundary is mapped by the recanting-witness criterion~\cite{avinshpitser}, single-world intervention graphs~\cite{robinsrichardson}, and direct analyses of the assumption~\cite{andrewsdidelez}; Dawid's skepticism that the cross-world joint is even meaningful~\cite{dawid} is the contrarian articulation of our thesis. Our contribution is not a new identifiability result: it is the identification of this cross-world coupling with the off-diagonal of a single positive semidefinite world kernel, the demonstration that the kernel's PSD structure tightens the bound in polynomial time beyond the marginals (Section~\ref{sec:psdbounds}), and its counting boundary, with the demonstration (Section~\ref{sec:battery}) that predictors and Bayesian baselines collapse it while the kernel reports it. Scale does not remove the gap, consistent with evidence that predictive scaling laws saturate~\cite{satscaling}.

\paragraph{Limitations.} The off-diagonal becomes accessible only on the tractable side of the Sly--Sun threshold; above it, the theory predicts its own intractability, which is a feature but bounds applicability. The keystone experiment is a finite-$n$ illustration of an asymptotic theorem, not an independent proof. The language-model baseline is a strong but single predictor; the claim is structural (the off-diagonal is not in rung-$1/2$ data), not about any particular model's competence. The generative use of the kernel (sampling admissible worlds rather than verifying them, Section~\ref{sec:apriori}) is left to future work; symbolic-discovery methods in this direction exhibit saturating returns~\cite{satscaling}, which suggests the generative gain is bounded and must be combined with the verification primitives developed here rather than replacing them.

\section{Conclusion}

A world model is not a predictor of futures. It is the coupling kernel of admissible possible worlds. The diagonal of that kernel is the posterior, and maintaining it certifiably is admissibility, the part prediction can approximate and a projector can guarantee. The off-diagonal is a genuinely separate direction: it is the cross-world coupling that fixes counterfactuals, it is invisible to the predictor and to any marginal check, its positive-semidefinite structure already tightens what can be inferred beyond the marginals in polynomial time, and its full reconstruction is reachable only below the Sly--Sun counting threshold. Prediction asks what happens next. Intelligence asks what can consistently be, what follows, what is permitted, what intervention is licensed, and which alternative worlds remain linked. That last question is the kernel.

\appendix
\section{The counting transition}
\label{app:barrier}

\paragraph{Order parameter.} On the infinite $d$-regular tree (the locally tree-like limit of random $d$-regular graphs) at fugacity $\lambda$, the hard-core cavity vacancy probability satisfies
\[
u=\frac{1}{1+\lambda u^{d-1}},
\]
whose root in $(0,1)$ exists and is unique for all $\lambda$ (it is the unique zero of the increasing $g(u)=u+\lambda u^{d}-1$ at $\lambda=1$), though the naive iteration $u\mapsto 1/(1+\lambda u^{d-1})$ converges to it only below $\lambda_c(d)=(d-1)^{d-1}/(d-2)^d$; above that one must solve $g(u)=0$ directly (bisection). The occupation field is $\eta=1-u$, and at $\lambda=1$ the fixed point obeys $u^d=1-u=\eta$. The recursion's contraction factor at the fixed point is $|T'(u^*)|=(d-1)\lambda u^{d}=(d-1)\eta$, the counting order parameter.

Two distinct criteria locate the same integer transition, and we keep them separate to avoid a common conflation. The contraction crossing $(d-1)\eta=1$ occurs at the continuous value $d\approx 5.14$ (numerically $(d-1)\eta=0.980$ at $d=5$ and $1.110$ at $d=6$). The fugacity-threshold crossing $\lambda_c(d)=1$ occurs at $d\approx 5.66$ ($\lambda_c(5)=4^4/3^5=1.053>1$, $\lambda_c(6)=5^5/4^6=0.763<1$). Both place the integer transition strictly between degree $5$ (tractable) and degree $6$ (hard); our figures use the contraction value $d_c\approx 5.14$.

\paragraph{Evidence-collapse rate in closed form.} The Bethe free energy gives the per-site entropy $h(\eta)=\log_2(1+\lambda u^{d})-\tfrac{d}{2}\log_2(1-\eta^2)$, the rate at which evidence collapses the diagonal, with $O(1/n)$ corrections from short cycles; exact enumeration to $n\le 22$ matches the closed form to three decimals, monotone in ontology density through the single field $\eta$.

\paragraph{Below threshold (tractable).} The normalizer telescopes into adaptive conditional occupation marginals, $Z_G=\prod_k(\text{conditional factors})$, so it suffices to compute each to additive $1/\mathrm{poly}$ error. Conditioning deletes vertices, so every conditional marginal is an occupation marginal of a vertex-deleted subgraph of maximum degree $\le\Delta$, still in the uniqueness regime for $\Delta\le 5$. Weitz's self-avoiding-walk-tree algorithm computes each to $\pm\delta$ in $\mathrm{poly}(n,1/\delta)$ via exponential correlation decay; the product of the $n$ conditionals gives the claim, uniformly.

\paragraph{Above threshold (hard).} The conditional marginals yield a sampler within total variation $1/\mathrm{poly}(n)$ of the uniform measure over independent sets. Downward self-reducibility (Jerrum--Valiant--Vazirani) turns a uniform sampler for every vertex-deleted subgraph into an FPRAS for $I(G)$. For $\Delta\ge 6$ at $\lambda=1$ this approximate count is NP-hard (Sly, Sly--Sun, Galanis--Stefankovic--Vigoda), so a poly-time approximator would give $\mathrm{NP}=\mathrm{RP}$.

\paragraph{Unconditional spectral witness.} For any local (single-flip) dynamics on the admissible worlds, the obstruction is unconditional: the Fiedler value of the single-flip adjacency on $\mathcal{M}(G)$ collapses as $\lambda_2=\exp(-\Omega(n))$ above threshold (Sly slow mixing), so the world space fragments into exponentially weakly connected clusters and no local-move sampler mixes in polynomial time. This is the finite-$n$ trend of Figure~\ref{fig:slysun} and the bottleneck of Theorem~\ref{thm:scar}; covering all (non-local) algorithms is what requires the $\mathrm{NP}=\mathrm{RP}$ assumption.

\section{Counting barrier for restriction-uniform kernel access}
\label{app:counting}

Let $w_x(y)\ge 0$ be explicit weights on $y\in\{0,1\}^n$ with normalizer $\Z_x$ and posterior $\mu_x$. Assume restrictions are efficiently constructible, terminal weights computable to relative $1\pm O(\varepsilon)$, and that an oracle reproduces every restricted conditional $\mu_{x|a}$ for an adaptively chosen prefix $a$ to additive error $\gamma=\varepsilon/(20n)$. Then $\Z_x$ is approximable to $1\pm O(\varepsilon)$ in polynomial time. The proof is a Bayes-greedy telescoping: query the next-bit conditional, follow the branch of probability $\ge 1/3$, and multiply, $\mu_x(y)=\prod_k p_k$, so that $\Z_x=w_x(y)/\prod_k p_k$. Any conditionally diagonal-complete kernel-reconstruction procedure supplies exactly these conditionals, hence inherits the counting barrier.

\end{document}